\documentclass{article} %
\usepackage[T1]{fontenc}
\usepackage{iclr2027_conference,times}

\usepackage{amsmath,amsfonts,bm}

\def\eqref#1{equation~\ref{#1}}
\def\1{\bm{1}}

\DeclareMathAlphabet{\mathsfit}{\encodingdefault}{\sfdefault}{m}{sl}
\SetMathAlphabet{\mathsfit}{bold}{\encodingdefault}{\sfdefault}{bx}{n}

\usepackage{algorithm}
\usepackage{algpseudocode}
\usepackage{amsthm}
\usepackage{multirow}
\usepackage{booktabs}
\usepackage{array}
\usepackage{colortbl}
\usepackage{xcolor}
\usepackage{placeins}
\usepackage{graphicx}
\usepackage{capt-of}
\usepackage{wrapfig}
\usepackage{hyperref}
\usepackage{url}

\newsavebox{\trainingtablebox}

\title{MORCA: Offline-to-Online Reinforcement Learning for Adaptive Cache Reuse in Video Diffusion Acceleration}

\author{\parbox{\dimexpr\textwidth-2\tabcolsep\relax}{\raggedright
Yuxiang Xiong\textsuperscript{1*}\quad
Ruiyan Wang\textsuperscript{1*}\quad
Wenqiang Wang\textsuperscript{2*}\quad
Teng Hu\textsuperscript{1\ensuremath{\ddagger}}\\[0.3ex]
Songhang Shen\textsuperscript{1}\quad
Bohao Feng\textsuperscript{3}\quad
Hongqian Deng\textsuperscript{2}\quad
Ran Yi\textsuperscript{1\ensuremath{\dagger}}\\[0.6ex]
\normalfont\normalsize
\textsuperscript{1}Shanghai Jiao Tong University\quad
\textsuperscript{2}Alibaba Cloud Computing\\
\textsuperscript{3}Alibaba Token Hub, Alibaba Group\\[0.5ex]
{\fontsize{8}{9.5}\selectfont
\texttt{\{xiongyuxiang, 22-wry, hu-teng, 1479014299, ranyi\}@sjtu.edu.cn}\\
\texttt{\{channing.wwq, fengbohao.fbh\}@alibabainc.com}\quad
\texttt{hongqiandeng@foxmail.com}}
}}

\theoremstyle{plain}
\newtheorem{proposition}{Proposition}[section]

\iclrfinalcopy
\begin{document}

\maketitle
\begingroup
\renewcommand{\thefootnote}{}
\begin{NoHyper}
\footnotetext{\textsuperscript{*}Equal contribution.\quad
\textsuperscript{\ensuremath{\dagger}}Corresponding author.\quad
\textsuperscript{\ensuremath{\ddagger}}Project leader.}
\end{NoHyper}
\endgroup

\begin{abstract}
Diffusion Transformers (DiTs) achieve remarkable performance in video synthesis, but their iterative denoising process suffers from high inference latency. To address this, caching has emerged as an effective acceleration strategy by capitalizing on %
inter-step redundancy during denoising.
 Existing dynamic caching methods typically estimate the error that cache reuse would introduce at each denoising step (step error) to guide cache decisions, whereas our concern is how much quality loss cache reuse would cause in the final generated video (terminal error). We show that step error does not directly correspond to terminal error and that latent information helps capture their relationship, thereby informing cache decisions. Moreover, existing threshold-based methods cannot provide precise speedup control, making it difficult to meet practical requirements for user-specified acceleration targets. To address these limitations, we introduce \textbf{\fontencoding{T1}\selectfont MORCA}, a cache scheduling framework trained through offline-to-online reinforcement learning to make latent-aware reuse/recompute decisions under user-specified acceleration targets. Extensive experiments on different video generation models across multiple target acceleration ratios demonstrate that MORCA achieves better generation fidelity than state-of-the-art caching methods under comparable computational budgets.
Code is available at \url{https://github.com/x10ngyx/MORCA}.
\end{abstract}

\section{Introduction}
\label{sec:intro}

\suppressfloats[t]
\newlength{\introtextfloatsep}
\setlength{\introtextfloatsep}{\textfloatsep}
\setlength{\textfloatsep}{6pt}
\AddToHook{shipout/after}[morca-intro-spacing]{%
  \ifnum\value{page}=2\relax
    \global\textfloatsep=\introtextfloatsep\relax
    \RemoveFromHook{shipout/after}[morca-intro-spacing]%
  \fi
}
\begin{figure}[t]
\centering
\includegraphics[width=\linewidth]{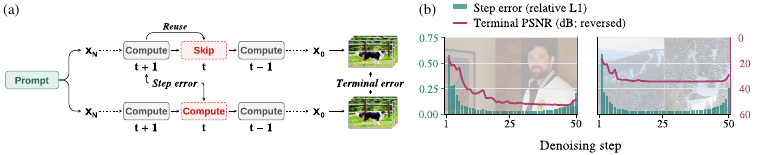}
\begingroup
\setlength{\abovecaptionskip}{1pt}
\caption{\textbf{\fontencoding{T1}\selectfont Step error and terminal error.} (a) Following SeaCache, we cache and reuse the residuals of the DiT blocks. For each step, we evaluate two cases: (i) reusing the residual cached at the previous step and (ii) recomputing the DiT blocks, with all other steps computed normally. We measure step error as the relative L1 difference between the cached and recomputed residuals, and terminal error using the PSNR of the resulting video with cache reuse relative to the no-cache reference. (b) Step error does not directly correspond to terminal error. Results are shown for Wan2.1.}
\label{fig:step-terminal-error}
\endgroup
\end{figure}

Diffusion Transformers (DiTs)~\citep{peebles2023scalable} combined with Flow Matching paradigms~\citep{lipman2023flow,liu2023flow} have demonstrated remarkable capabilities in high-fidelity visual generation~\citep{wan2025wan,kong2024hunyuanvideo}. Despite this progress, the multi-step iterative denoising process incurs substantial computational cost and high inference latency, particularly for video synthesis. To address this bottleneck, various acceleration approaches have emerged, including distillation, quantization, sparse attention, and caching. Among these, distillation requires additional training of the generative model, while quantization and sparse attention typically rely on specialized kernels. Caching offers a lightweight alternative with flexible speed--quality trade-offs.

Cache-based methods exploit inter-step redundancy by skipping the computation of the entire denoising network or selected components (typically the DiT blocks) at some denoising steps and reusing the corresponding results cached at previous steps. Early methods adopt fixed cache schedules~\citep{ma2024deepcache,li2024fasterdiffusion,zhao2025realtime,liu2025taylorseer} (e.g., recomputing only at fixed intervals), which do not account for runtime denoising dynamics. Dynamic caching methods~\citep{liu2025teacache,kahatapitiya2025adacache,chung2026seacache,bu2026dicache} attempt to estimate the error that cache reuse would introduce at each denoising step and use these estimates to make cache decisions through threshold-based rules.

However, this paradigm relies exclusively on the immediate error introduced by cache reuse (\textbf{\fontencoding{T1}\selectfont step error}) to make decisions, whereas our concern is the quality loss in the final generated video (\textbf{\fontencoding{T1}\selectfont terminal error}). How the former translates into the latter through subsequent denoising steps remains unclear. To investigate this relationship, we conduct a series of experiments. As shown in Fig.~\ref{fig:step-terminal-error}, step error exhibits similar patterns across these prompts, while terminal error varies substantially. Quantitative results in Table~\ref{tab:terminal-error-prediction} further show that step error alone is insufficient to accurately predict terminal error.

\noindent
\begin{minipage}[t]{\dimexpr0.5\linewidth-0.5\columnsep\relax}
\vspace{0pt}
We propose that latents may contain information that helps capture how step error translates into terminal error, thus informing cache decisions. To test this idea, we train a predictor to estimate terminal error from step error. As reported in Table~\ref{tab:terminal-error-prediction}, incorporating lightweight features extracted from the input latents at the current and previous steps (see Sec.~\ref{sec:scheduler}) substantially improves prediction accuracy.
\end{minipage}\hfill
\begin{minipage}[t]{\dimexpr0.5\linewidth-0.5\columnsep\relax}
\vspace{0pt}
\setlength{\abovecaptionskip}{0pt}
\setlength{\belowcaptionskip}{5pt}
\captionof{table}{\textbf{\fontencoding{T1}\selectfont Latent features improve terminal-error prediction.} Terminal error is measured using PSNR. Results on Wan2.1.}
\label{tab:terminal-error-prediction}
\fontsize{8}{9.2}\selectfont %
\setlength{\tabcolsep}{2pt}
\centering
\begin{tabular*}{\linewidth}{@{\extracolsep{\fill}}>{\raggedright\arraybackslash}p{\dimexpr0.54\linewidth-2pt\relax}rr@{}}
\toprule
Input & RMSE (dB) $\downarrow$ & Spearman $\uparrow$ \\
\midrule
Step error & 6.971 & 0.5288 \\
Step error + \mbox{step index} & 4.421 & 0.7528 \\
Step error + \mbox{latent} & 2.849 & 0.9357 \\
Step error + \mbox{latent} + \mbox{step index} & 2.832 & 0.9361 \\
\bottomrule
\end{tabular*}
\end{minipage}
\par

Motivated by this finding, we train a latent-aware scheduler to make terminal-error-guided cache decisions. We provide the scheduler with SeaCache's proxy signal for step error (see Sec.~\ref{sec:method-preliminaries}) and latent features to help bridge the gap between step error and terminal error. Given the sequential nature of cache scheduling, we formulate it as a Markov decision process (MDP) and optimize the scheduler using reinforcement learning. In practical applications, users often require a specific acceleration ratio. While existing threshold-based caching methods offer only approximate speedup control through threshold adjustment, we use the observed relationship between speedup and the number of reuse steps (Fig.~\ref{fig:calibration}) to translate the target acceleration ratio into a reuse budget constraint, then reduce the resulting constrained MDP to an equivalent standard augmented MDP (Sec.~\ref{sec:scheduler}).

Since collecting training data requires costly inference runs, we adopt an offline-to-online reinforcement learning framework based on IQL~\citep{kostrikov2022offline} to train the scheduler. As an off-policy algorithm, IQL offers higher data utilization than on-policy methods. Offline training provides an effective and stable initialization from limited data, while online fine-tuning allows the scheduler to explore under its current policy and further improve performance. Our contributions are summarized as follows:
\begin{list}{\textbullet}{%
    \setlength{\leftmargin}{0.9em}%
    \setlength{\rightmargin}{0pt}%
    \setlength{\labelwidth}{0.4em}%
    \setlength{\labelsep}{0.5em}%
    \setlength{\itemindent}{0pt}%
    \setlength{\listparindent}{0pt}%
    \setlength{\itemsep}{\parskip}%
    \setlength{\topsep}{0pt}%
    \setlength{\parsep}{0pt}%
    \setlength{\partopsep}{0pt}%
    \renewcommand{\makelabel}[1]{\hfil#1}}
    \item We propose \textbf{MORCA}, a latent-aware cache scheduling framework trained through offline-to-online reinforcement learning to make adaptive cache decisions under user-specified speedups.
\item We design a \textbf{budget-constrained} and \textbf{latent-aware} scheduler that can make terminal-error-guided cache decisions while precisely controlling the acceleration ratio.
    \item We introduce an \textbf{offline-to-online} RL paradigm for scheduler optimization: the offline stage provides an effective initialization, while online fine-tuning further improves its performance.
    \item Extensive experiments on different video generation models demonstrate that MORCA achieves better speed--quality trade-offs than state-of-the-art caching methods.
\end{list}

\section{Related Work}

\subsection{Diffusion Inference Acceleration}
\label{sec:diffusion-inference-acceleration}

Existing approaches to accelerating diffusion inference mainly include fast sampling, distillation, quantization, sparse attention, and caching. Fast samplers and numerical solvers reduce model evaluations but may sacrifice quality under aggressive step reduction~\citep{song2021denoising,liu2022pseudo,NEURIPS2022_260a14ac,Lu2025,NEURIPS2023_9c2aa1e4}. Distillation enables few-step or one-step generation but requires additional training~\citep{salimans2022progressive,pmlr-v202-song23a,luo2023latentconsistencymodelssynthesizing,Yin_2024_CVPR}. Quantization reduces the precision of weights and activations~\citep{Shang_2023_CVPR,Li_2023_ICCV,NEURIPS2023_2aab8a76,ICLR2025_a4a1ee07}, while sparse attention restricts attention computation to selected token pairs~\citep{pmlr-v267-xi25c}. Both effectively lower per-step computation but typically rely on specialized kernels to realize inference speedups.

\subsection{Caching-based Diffusion Model Acceleration}
Cache-based acceleration methods exploit inter-step redundancy by skipping computation in all or part of the denoising network at selected steps and reusing previously cached results. These methods can be characterized along three complementary dimensions. \emph{Caching Granularity} determines what is cached and reused, including layer or DiT-block outputs~\citep{wimbauer2024cacheme,chen2024deltadit}, attention outputs~\citep{zhao2025realtime}, CFG branch outputs~\citep{lv2024fastercache}, and finer-grained token features~\citep{zou2025toca}. \emph{Cache Reuse Strategy} determines how cached features are utilized, evolving from direct reuse~\citep{ma2024deepcache,zhao2025realtime} to feature correction or approximation~\citep{lv2024fastercache,bu2026dicache}, and further to explicit feature forecasting~\citep{liu2025taylorseer}. \emph{Caching Schedule} determines which steps reuse cached results and which recompute. Early methods typically employ fixed schedules with predefined recomputation steps~\citep{ma2024deepcache,li2024fasterdiffusion,zhao2025realtime,liu2025taylorseer}, which do not adapt to runtime denoising dynamics. Dynamic methods instead estimate the error that cache reuse would introduce at each step and apply threshold-based rules to guide cache decisions. TeaCache uses timestep-modulated input differences as a proxy for output variation, AdaCache further incorporates motion-aware information, DiCache estimates caching errors through shallow-layer probes, and SeaCache derives a spectral-evolution-aware cache signal by suppressing noise-dominated components~\citep{liu2025teacache,kahatapitiya2025adacache,bu2026dicache,chung2026seacache}. Recent works have also explored learning-based cache scheduling~\citep{zhao2026rapid3,xie2026onlinecache,zhao2026bag}, which we discuss and compare in detail in Sec.~\ref{sec:learning-based-cache-scheduling}.

\subsection{Reinforcement-Learning-Based Cache Scheduling}
\label{sec:learning-based-cache-scheduling}

Cache scheduling is naturally a sequential decision problem suited to reinforcement learning. RAPID$^3$ is the first work to formulate cache scheduling as an MDP and train a scheduler through reinforcement learning to make cache-reuse decisions at each denoising step~\citep{zhao2026rapid3}. Concurrent with our work, BAG and OnlineCache explore different approaches to learning cache schedules. BAG constructs a high-quality schedule dataset through expensive offline search and distills these schedules into a lightweight gating network for per-step decisions~\citep{zhao2026bag}. OnlineCache uses REINFORCE to train a latent-aware scheduler, but its on-policy RL paradigm limits data reuse and training flexibility. Moreover, its reward-based speedup control cannot guarantee meeting user-specified acceleration targets~\citep{xie2026onlinecache}.

\section{Method}
\label{sec:method}

In this section, we present MORCA, an offline-to-online reinforcement learning framework for budget-constrained and latent-aware cache scheduling in video diffusion acceleration, as illustrated in Fig.~\ref{fig:pipeline}. Building upon SeaCache (Sec.~\ref{sec:method-preliminaries}), MORCA retains its caching granularity and cache reuse strategy while replacing its threshold-based decision rule with a learned scheduler. In addition to SeaCache's signals and related statistics, the scheduler incorporates latent features and reuse budget dynamics to make terminal-error-guided cache decisions under user-specified acceleration targets (Sec.~\ref{sec:scheduler}). We optimize the scheduler through a two-stage offline-to-online IQL paradigm to achieve both high sample efficiency and strong performance (Sec.~\ref{sec:training}).

\begin{figure}[t]
\centering
\includegraphics[width=\linewidth]{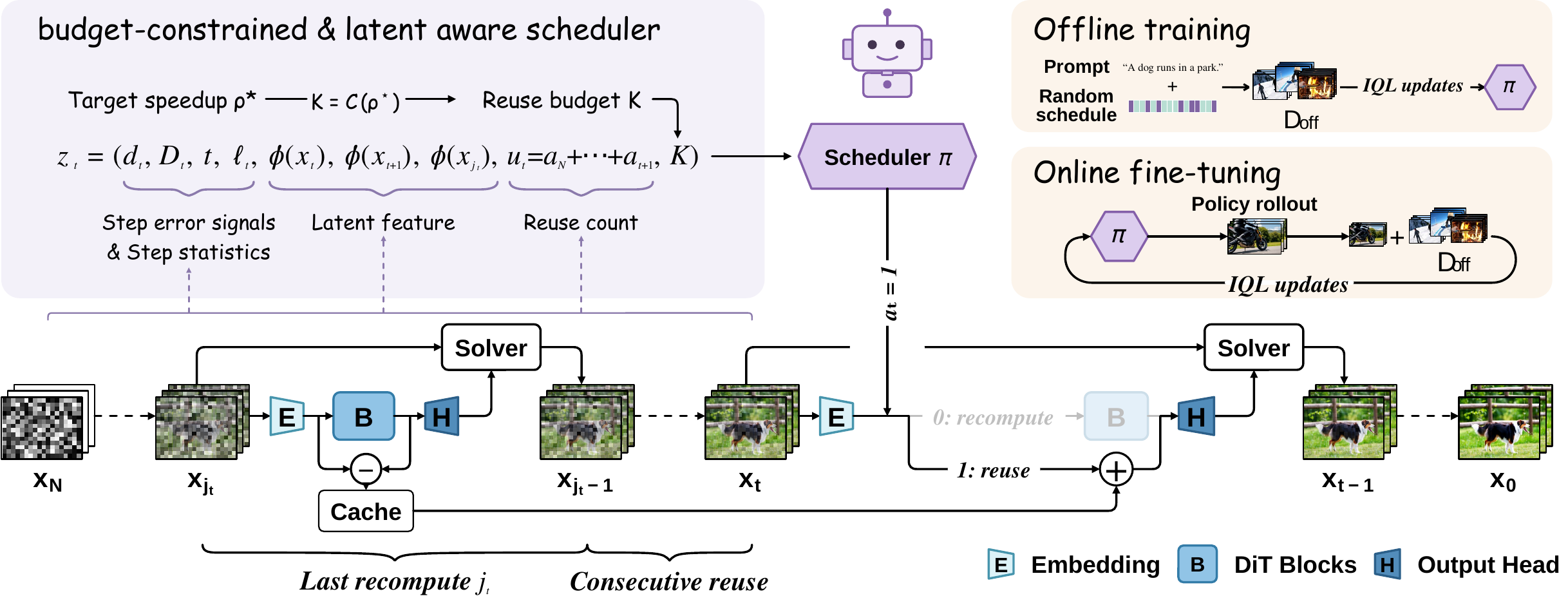}
\caption{\textbf{\fontencoding{T1}\selectfont Overview of MORCA.} Bottom: A learned scheduler makes reuse/recompute decisions at each denoising step. Top left: The scheduler is budget-constrained and latent-aware. The target speedup is mapped to the required total number of reuse steps, which, together with step-error proxies and step statistics, latent features, and the current reuse count extracted from the denoising state, is fed into the scheduler for decision-making (Sec.~\ref{sec:scheduler}). Top right: The scheduler is optimized through offline-to-online reinforcement learning using IQL (Sec.~\ref{sec:training}).}
\label{fig:pipeline}
\end{figure}

\subsection{Preliminaries}
\label{sec:method-preliminaries}

\paragraph{Cache-Based Inference.}
We first introduce the standard denoising process, followed by a typical cache-based inference pipeline. Given a text condition $c$, video generation proceeds through $N$ denoising steps. At step $t$, the input latent $x_t$ passes through a DiT denoiser comprising an input embedding module $E_\theta$, a stack of DiT blocks $B_\theta$, and an output head $H_\theta$:
\begin{equation}
h_t=E_\theta(x_t),\qquad
y_t=B_\theta(h_t,t,c),\qquad
o_t=H_\theta(y_t,t).
\label{eq:method-denoiser-components}
\end{equation}
The latent is updated as $x_{t-1}=\mathcal{U}_t(x_t,o_t)$, where $\mathcal{U}_t$ denotes the diffusion- or flow-based sampling update rule. With classifier-free guidance (CFG), $o_t$ is replaced by $o_t^{\mathrm{CFG}}=o_t^{\mathrm{u}}+w(o_t-o_t^{\mathrm{u}})$, where $o_t^{\mathrm{u}}$ is computed using the empty or negative text condition $c_{\mathrm{u}}$, and $w$ is the guidance scale. After the final denoising step, the latent $x_0$ is decoded into a video by the VAE decoder.

Cache-based inference decides whether to recompute or reuse at each denoising step. A binary action $a_t\in\{0,1\}$ determines whether to skip the DiT blocks and directly reuse a previously cached residual ($a_t=1$) or compute the blocks normally ($a_t=0$). The block residual at step $t$ is defined as $\Delta_t=y_t-h_t$. Accordingly,
\begin{equation}
h_t=E_\theta(x_t),\qquad
\hat{y}_t=
\begin{cases}
B_\theta(h_t,t,c), & a_t=0,\\
h_t+\Delta_{j_t}, & a_t=1,
\end{cases}
\qquad
\hat{o}_t=H_\theta(\hat{y}_t,t),
\label{eq:method-cache-output}
\end{equation}
where $j_t=\min\{j>t:a_j=0\}$ is the most recent recomputation step.

Under CFG, the two branches typically share $a_t$ but maintain separate residual caches. Their outputs are combined using the same guidance rule for the sampling update. The complete cache schedule is $\mathbf{a}=(a_N,\ldots,a_1)$.

\paragraph{SeaCache.}
SeaCache~\citep{chung2026seacache} applies a Spectral-Evolution-Aware (SEA) filter to the timestep-modulated input latent and uses the filtered representation, denoted by $\widetilde{x}_t$, to construct a lightweight proxy signal for estimating the relative change between adjacent steps:
\begin{equation}
d_t=
\frac{\lVert\widetilde{x}_t-\widetilde{x}_{t+1}\rVert_1}
{\lVert\widetilde{x}_{t+1}\rVert_1+\epsilon}.
\label{eq:method-seacache-distance}
\end{equation}
It further computes the accumulated distance since the most recent recomputation step as $D_t=\sum_{i=t}^{j_t-1}d_i$ and determines the action using a threshold $\delta$: $a_t^{\mathrm{SC}}=\mathbb{I}[D_t\leq\delta]$.

\subsection{Budget-Constrained and Latent-Aware Scheduler}
\label{sec:scheduler}

We formulate cache scheduling as a budget-constrained MDP and reduce it to an equivalent standard augmented MDP compatible with standard RL optimization. We then introduce a compact latent-aware state representation for efficient scheduler implementation.

\paragraph{Constrained MDP Formulation.}
Cache scheduling is inherently a sequential decision problem, as each recompute/reuse decision affects subsequent denoising states and therefore has long-term consequences. We formulate it as a finite-horizon MDP $\mathcal{M}=(\mathcal{S},\mathcal{A},P,R,N)$. At each step $t$, $s_t\in\mathcal{S}$ denotes the denoising state, and $a_t\in\mathcal{A}=\{0,1\}$ denotes the action, with $0$ corresponding to recomputation and $1$ to cache reuse. The transition dynamics $P$ describe how the denoising state evolves under the selected action through the corresponding denoising update. We define the reward $R$ in terms of the generated video's quality, measured by its PSNR relative to the no-cache reference $y^\star$ generated with the same prompt and random seed:
\begin{equation}
r_t=
\begin{cases}
0, & 1<t\le N,\\
\min\!\left\{\operatorname{PSNR}(y,y^\star),\,r_{\max}\right\}, & t=1,
\end{cases}
\qquad
R(\tau)=\sum_{t=1}^{N}r_t.
\label{eq:terminal-reward}
\end{equation}
where $y$ denotes the video generated under the cache schedule and $r_{\max}$ is a fixed finite cap on PSNR. The horizon $N$ is the number of denoising steps.

\begin{wrapfigure}{r}{\dimexpr0.5\linewidth-0.5\columnsep\relax}
\vspace{-12pt}
\centering
\includegraphics[width=\linewidth,keepaspectratio]{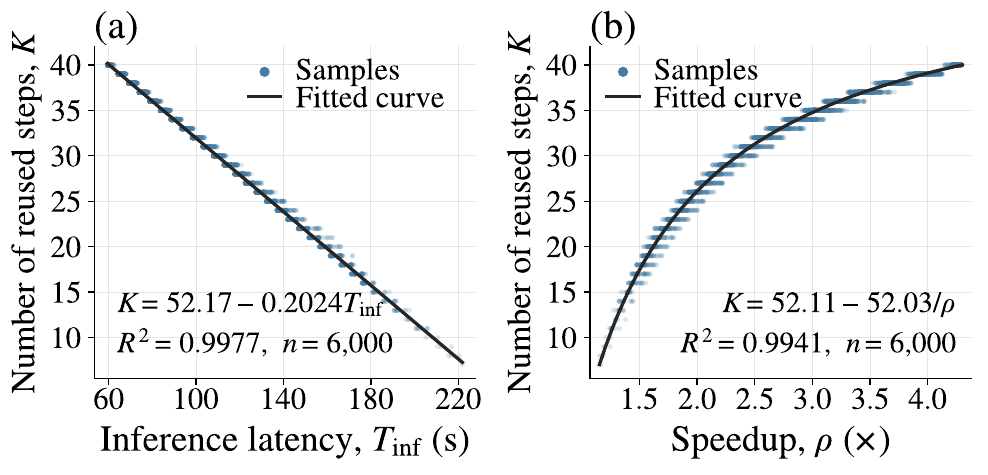}
\setlength{\abovecaptionskip}{0pt}
\caption{\textbf{\fontencoding{T1}\selectfont Calibration of reuse budget $K$ against (a) inference latency and (b) speedup on Wan2.1.} The fitted functions take the forms (a) $K=\alpha-\beta T_{\mathrm{inf}}$ and (b) $K=\alpha-\beta/\rho$.}
\label{fig:calibration}
\end{wrapfigure}

Given a target acceleration ratio $\rho^\star$, we further constrain the computation budget. Since all denoising steps share the same backbone, each reuse operation yields approximately the same computational saving. Empirically, inference latency decreases approximately linearly with the reuse count $K$ and is similar across prompts. Combined with $\rho=T_{\mathrm{full}}/T_{\mathrm{inf}}$, this yields $K\approx\alpha-\beta/\rho$ (Fig.~\ref{fig:calibration}). We therefore calibrate the target acceleration ratio to a reuse budget $K=\mathcal{C}(\rho^\star)$ and formulate cache scheduling as a constrained MDP. We seek to learn a scheduler policy $\pi$ that maximizes the expected return while satisfying the prescribed reuse budget:
\par
\makeatletter\WFclear\makeatother
\begin{equation}
\max_{\pi}\; \mathbb{E}_{\tau\sim\pi}[R(\tau)]
\qquad \mathrm{s.t.}\qquad
\sum_{t=1}^{N}a_t=K.
\label{eq:hard-budget-objective}
\end{equation}

\paragraph{Reduction to a Standard MDP.}
However, RL algorithms generally do not directly guarantee constraint satisfaction. To make the constrained scheduling problem compatible with standard RL optimization, we transform it into an equivalent augmented MDP following the reduction approach of \citet{mcmahan2024anytime}. Let $u_t=\sum_{i=t+1}^{N}a_i$ denote the number of reuse actions already taken and $b_t=K-u_t$ the remaining budget. Since the optimal decision depends on both the current denoising state and the remaining budget, we augment the state as
\begin{equation}
\bar{s}_t=(s_t,u_t,K).
\label{eq:augmented-state}
\end{equation}

We further restrict the action set to prevent decisions that make the prescribed reuse budget unattainable. After taking action $a_t$, the remaining budget $b_t-a_t$ must be achievable within the remaining $t-1$ denoising steps. We therefore define the feasible action set as
\begin{equation}
\bar{\mathcal{A}}_t(\bar{s}_t)
=
\left\{
a\in\{0,1\}
\;\middle|\;
0\le b_t-a\le t-1
\right\}.
\label{eq:safe-action-set}
\end{equation}
This construction yields a standard augmented MDP $\bar{\mathcal{M}}$ that preserves the optimal feasible value and guarantees the prescribed reuse budget. A formal proof is provided in Appendix~\ref{app:budget-proof}.

\paragraph{Latent-Aware State Representation.}
In the formulation and reduction above, $\bar{s}_t$ denotes the full denoising state needed to determine transition and reward distributions. In practice, we use a simplified representation of $\bar{s}_t$ as the scheduler input to control computational overhead. Recall from Sec.~\ref{sec:intro} that our scheduler aims to make terminal-error-guided decisions. Existing methods already provide effective proxy signals for estimating step error, but this local error does not directly correspond to terminal error (Fig.~\ref{fig:step-terminal-error}). Latent information helps bridge this gap (Table~\ref{tab:terminal-error-prediction}).

Our scheduler input therefore consists of three components. (i) Step-error proxies and related statistics: we retain signals and statistics associated with SeaCache's cache decisions, including the adjacent-step distance $d_t$, accumulated distance $D_t$, current denoising step $t$, and number of consecutive reuse steps $\ell_t=j_t-t-1$. (ii) Latent information: each latent has shape $C\times T\times H\times W$, where $C$, $T$, $H$, and $W$ denote the channel count, latent frame count, height, and width, respectively. To control computational overhead, we extract compact features instead of using the full latent. We apply 3D pooling to each latent and 2D pooling to its temporal mean and variance. Denoting these pooling operators by $\mathcal{P}_3$ and $\mathcal{P}_2$, and the temporal mean and variance by $\mu_T(x)$ and $\sigma_T^2(x)$, respectively, we define $\phi(x)=\bigl(\mathcal{P}_3(x),\;\mathcal{P}_2(\mu_T(x)),\;\mathcal{P}_2(\sigma_T^2(x))\bigr)$. We use $\phi$ to extract features from latents relevant to the current denoising state: the current latent $x_t$, the previous-step latent $x_{t+1}$, and the latent at the most recent recomputation step $x_{j_t}$. (iii) Budget information: following the reduction above, we include $u_t$ and $K$ to handle the budget constraint. Together, these components form the compact scheduler input:
\begin{equation}
z_t=\left(d_t,D_t,t,\ell_t,\phi(x_t),\phi(x_{t+1}),\phi(x_{j_t}),u_t,K\right).
\label{eq:latent-aware-scheduler-input}
\end{equation}

\subsection{Offline-to-Online Reinforcement Learning Paradigm}
\label{sec:training}

Since collecting cache-scheduling trajectories requires expensive video-generation rollouts, we adopt a two-stage offline-to-online reinforcement learning paradigm based on Implicit Q-Learning (IQL)~\citep{kostrikov2022offline} to achieve both high sample efficiency and strong performance. As an off-policy algorithm, IQL enables repeated reuse of collected data across optimization updates, unlike on-policy methods such as REINFORCE, which require fresh rollouts under the current policy. In the offline stage, the scheduler learns from a relatively small set of pre-collected random trajectories, providing an effective policy initialization. Starting from this initialization, online fine-tuning enables further exploration and policy improvement.

\paragraph{Offline Training.}
We randomly sample cache schedules under different prompts and target acceleration ratios, run the video-generation model with these schedules, and convert the resulting trajectories into transition tuples $(z_t,a_t,r_t,z_{t-1})$ to construct the training dataset $\mathcal D_{\mathrm{off}}$. Further details of dataset construction are provided in Appendix~\ref{app:dataset-preparation}.

As illustrated in Algorithm~\ref{alg:iql-training}(a), IQL maintains a value network $V_\psi$, two Q-networks $Q_{\theta_1}$ and $Q_{\theta_2}$, and their corresponding target networks $Q_{\bar\theta_1}$ and $Q_{\bar\theta_2}$. The target networks track the Q-networks through exponential moving average (EMA) updates. IQL trains the Q-functions against temporal-difference targets and learns the value function through expectile regression on the target Q-values:
\begin{equation}
\mathcal L_Q(\theta_i)
=
\mathbb E_{(z_t,a_t,r_t,z_{t-1})\sim\mathcal D}
\left[\left(r_t+\gamma V_\psi(z_{t-1})-Q_{\theta_i}(z_t,a_t)\right)^2\right],
\qquad i=1,2.
\label{eq:iql-q-loss}
\end{equation}
\begin{equation}
\mathcal L_V(\psi)
=
\mathbb E_{(z_t,a_t)\sim\mathcal D}
\left[\ell_\tau\left(\bar Q(z_t,a_t)-V_\psi(z_t)\right)\right].
\label{eq:iql-value-loss}
\end{equation}
where $\ell_\tau(u)=|\tau-\mathbb I[u<0]|u^2$ is the expectile loss, $\bar Q(z_t,a_t)=\min_{i\in\{1,2\}}Q_{\bar\theta_i}(z_t,a_t)$, and $\gamma$ is the discount factor.
The scheduler policy is then learned through advantage-weighted behavior cloning, assigning larger weights to actions with higher estimated advantages:
\begin{equation}
\mathcal L_\pi(\varphi) = -\mathbb E_{(z_t,a_t)\sim\mathcal D}
\left[
\exp\left(\beta A(z_t,a_t)\right)
\log\pi_\varphi(a_t\mid z_t)
\right],
\label{eq:iql-policy-loss}
\end{equation}
where $A(z_t,a_t)=\bar Q(z_t,a_t)-V_\psi(z_t)$, and $\beta$ controls advantage weighting. For further details of IQL, we refer readers to \citet{kostrikov2022offline}.

By learning preferentially from higher-advantage actions in a relatively small dataset of randomly sampled trajectories, the offline stage provides a good initial policy without requiring prior knowledge of effective cache schedules. In contrast, starting online training from scratch may incur substantial unnecessary data collection and training costs under an initially weak policy.

\paragraph{Online Fine-Tuning.}
Starting from the offline-trained scheduler, we perform online fine-tuning as shown in Algorithm~\ref{alg:iql-training}(b). In each round, we collect a batch of online trajectories under randomly sampled prompts and target acceleration ratios, sampling each step's cache decision from the current scheduler policy $\pi_\varphi(\cdot\mid z_t)$. These trajectories are added to the replay buffer alongside the offline and previously collected online data. We first perform a short critic-only warm-up with the policy fixed to adapt the Q- and V-networks to the new data, then jointly train the value functions and policy using the same IQL procedure as in offline training. Sampling trajectories from the current policy enables further exploration, while retaining the offline data in the replay buffer helps stabilize learning during fine-tuning.

\begin{algorithm}[t]
\caption{Offline-to-Online IQL Training for MORCA}
\label{alg:iql-training}
\small
\noindent
\begin{minipage}[t]{0.48\linewidth}
\textbf{(a) Offline training}
\begin{algorithmic}[1]
\State Randomly initialize $\psi,\theta_1,\theta_2,\varphi$
\State $\bar\theta_i\gets\theta_i,\quad i=1,2$
\State $\mathcal D\gets\mathcal D_{\mathrm{off}}$
\For{each gradient step}
    \State \Call{CriticUpdate}{}
    \State \Call{ActorUpdate}{}
\EndFor
\end{algorithmic}
\par\medskip\hrule\smallskip
\textbf{Shared updates}
\begin{algorithmic}
\Procedure{CriticUpdate}{}
    \State $\psi\gets\psi-\lambda_V\nabla_\psi\mathcal L_V(\psi)$
    \State $\theta_i\gets\theta_i-\lambda_Q\nabla_{\theta_i}\mathcal L_Q(\theta_i),\quad i=1,2$
    \State $\bar\theta_i\gets\eta\bar\theta_i+(1-\eta)\theta_i,\quad i=1,2$
\EndProcedure
\Procedure{ActorUpdate}{}
    \State $\varphi\gets\varphi-\lambda_\pi\nabla_\varphi\mathcal L_\pi(\varphi)$
\EndProcedure
\end{algorithmic}
\end{minipage}\hfill
\begin{minipage}[t]{0.48\linewidth}
\textbf{(b) Online fine-tuning}
\begin{algorithmic}[1]
\State $\mathcal D\gets\mathcal D_{\mathrm{off}}$
\State Initialize $\psi,\theta_1,\theta_2,\bar\theta_1,\bar\theta_2,\varphi$ from the offline training results
\For{$k=1,\ldots,N_{\mathrm{on}}$}
    \State Collect new data $\Delta\mathcal D$ using $\pi_\varphi$
    \State $\mathcal D\gets\mathcal D\cup\Delta\mathcal D$
    \For{each gradient step}\Comment{warm-up}
        \State \Call{CriticUpdate}{}
    \EndFor
    \For{each gradient step}\Comment{joint}
        \State \Call{CriticUpdate}{}
        \State \Call{ActorUpdate}{}
    \EndFor
\EndFor
\State \Return $\pi_\varphi$
\end{algorithmic}
\end{minipage}

\par\smallskip\hrule\smallskip\noindent\footnotesize
$N_{\mathrm{on}}$ is the number of online rounds; $\lambda_V,\lambda_Q,\lambda_\pi$ are learning rates; $\eta$ is the target-network EMA decay.
\end{algorithm}

\section{Experiments}

\subsection{Experimental Settings}

{\fontencoding{T1}\selectfont\bfseries Model configurations.} We evaluate two pretrained text-to-video models from the Wan family~\citep{wan2025wan}: Wan2.2-T2V-A14B on NVIDIA RTX A6000 GPUs and Wan2.1-T2V-1.3B on NVIDIA GeForce RTX 4090 GPUs. Both use 50 denoising steps; other inference settings are provided in Appendix~\ref{app:model-configurations}.

\label{sec:training-configurations}
{\fontencoding{T1}\selectfont\bfseries Training configurations.} We collect 2,500 offline trajectories from 100 prompts for Wan2.2 (25 schedules per prompt), with target speedups in $[1.5\times,3.5\times]$, and 4,000 from 1,000 prompts for Wan2.1 (four schedules per prompt), with target speedups in $[2.0\times,4.5\times]$. Offline IQL runs for 200 epochs with a learning rate of $10^{-4}$ for all networks. Online fine-tuning runs for 16 rounds, each adding 100 trajectories and performing five epochs of critic-only warm-up followed by 20 epochs of joint training, with learning rates of $4\times10^{-5}$ for the policy network and $10^{-4}$ for the Q and V networks. Both stages use the AdamW optimizer and a batch size of 256. Further details are provided in Appendix~\ref{app:training-details}.

{\fontencoding{T1}\selectfont\bfseries Baseline configurations.} We compare against four representative caching-based acceleration methods: TeaCache~\citep{liu2025teacache}, MagCache~\citep{ma2025magcache}, DiCache~\citep{bu2026dicache}, and SeaCache~\citep{chung2026seacache}. For Wan2.2, only MagCache has an official implementation; we port the others from Wan2.1 with model-specific adaptations. We use their official Wan2.1 implementations directly. Implementation details are provided in Appendix~\ref{app:baseline-details}.

{\fontencoding{T1}\selectfont\bfseries Evaluation protocol.} We sample 200 prompts from VBench~\citep{huang2024vbench} as a fixed evaluation set and measure visual quality using PSNR, SSIM, and LPIPS against no-cache references generated with the same prompts and random seeds. We report inference latency. We compare methods at approximately $1.8\times$, $2.4\times$, and $3.0\times$ speedup for Wan2.2, and $2.4\times$, $3.0\times$, and $4.0\times$ for Wan2.1. We use lower target speedups for Wan2.2 to preserve generation fidelity, as it is more sensitive to cache reuse.

\subsection{Quantitative Comparison}
\label{sec:quantitative-comparison}

\begin{table}[t]
\fontencoding{T1}\selectfont
\setlength{\belowcaptionskip}{5pt}
\caption{\textbf{\fontencoding{T1}\selectfont Quantitative comparison with baselines.} The best result is highlighted in bold, while the second-best result is underlined.}
\label{tab:quality-main}
\centering
\fontsize{8}{9.2}\selectfont %
\setlength{\tabcolsep}{8pt}
\renewcommand{\arraystretch}{1}
\setlength{\aboverulesep}{1.5pt}
\setlength{\belowrulesep}{1.5pt}
\begin{tabular}{cccccccc}
\toprule[0.7pt]
\textbf{Model} & \textbf{Target} & \textbf{Method} & \mbox{\textbf{Latency} (s)}$\downarrow$ & \textbf{Speedup}$\uparrow$ & \textbf{PSNR}$\uparrow$ & \textbf{SSIM}$\uparrow$ & \textbf{LPIPS}$\downarrow$ \\
\midrule[0.35pt]
\multirow{16}{*}[-4pt]{Wan2.2} & --- & Original & 990.56 & $1.00\times$ & --- & --- & --- \\
\cmidrule[0.3pt](lr){2-8}
 & \multirow{5}{*}{$1.8\times$} & TeaCache & 562.63 & $1.76\times$ & 17.04 & 0.5984 & 0.2882 \\
 &  & MagCache & 544.78 & $1.82\times$ & 16.41 & 0.5662 & 0.3041 \\
 &  & DiCache & 551.29 & $1.80\times$ & 18.51 & 0.6516 & 0.2232 \\
 &  & SeaCache & 546.99 & $1.81\times$ & \underline{19.91} & \underline{0.7016} & \underline{0.1888} \\
 &  & \cellcolor{black!5}\textbf{Ours} & \cellcolor{black!5}549.19 & \cellcolor{black!5}$1.80\times$ & \cellcolor{black!5}\textbf{22.13} & \cellcolor{black!5}\textbf{0.7507} & \cellcolor{black!5}\textbf{0.1681} \\
\cmidrule[0.3pt](lr){2-8}
 & \multirow{5}{*}{$2.4\times$} & TeaCache & 415.21 & $2.39\times$ & 15.31 & 0.5325 & 0.3658 \\
 &  & MagCache & 417.22 & $2.37\times$ & 15.31 & 0.5185 & 0.3632 \\
 &  & DiCache & 416.43 & $2.38\times$ & 16.51 & 0.5759 & 0.2923 \\
 &  & SeaCache & 419.76 & $2.36\times$ & \underline{17.38} & \underline{0.6074} & \underline{0.2840} \\
 &  & \cellcolor{black!5}\textbf{Ours} & \cellcolor{black!5}421.15 & \cellcolor{black!5}$2.35\times$ & \cellcolor{black!5}\textbf{19.24} & \cellcolor{black!5}\textbf{0.6516} & \cellcolor{black!5}\textbf{0.2513} \\
\cmidrule[0.3pt](lr){2-8}
 & \multirow{5}{*}{$3.0\times$} & TeaCache & 322.88 & $3.07\times$ & 14.40 & 0.4913 & 0.4211 \\
 &  & MagCache & 332.74 & $2.98\times$ & 13.41 & 0.4336 & 0.4839 \\
 &  & DiCache & 340.50 & $2.91\times$ & \underline{15.54} & \underline{0.5259} & \underline{0.3623} \\
 &  & SeaCache & 320.41 & $3.09\times$ & 14.99 & 0.5151 & 0.3941 \\
 &  & \cellcolor{black!5}\textbf{Ours} & \cellcolor{black!5}326.93 & \cellcolor{black!5}$3.03\times$ & \cellcolor{black!5}\textbf{16.77} & \cellcolor{black!5}\textbf{0.5639} & \cellcolor{black!5}\textbf{0.3435} \\
\midrule[0.5pt]
\multirow{16}{*}[-4pt]{Wan2.1} & --- & Original & 255.94 & $1.00\times$ & --- & --- & --- \\
\cmidrule[0.3pt](lr){2-8}
 & \multirow{5}{*}{$2.4\times$} & TeaCache & 106.82 & $2.40\times$ & 20.12 & 0.7361 & 0.1596 \\
 &  & MagCache & 107.54 & $2.38\times$ & 22.31 & 0.8075 & 0.1198 \\
 &  & DiCache & 108.27 & $2.36\times$ & 22.28 & 0.8139 & \textbf{0.0992} \\
 &  & SeaCache & 109.00 & $2.35\times$ & \underline{22.86} & \underline{0.8174} & 0.1131 \\
 &  & \cellcolor{black!5}\textbf{Ours} & \cellcolor{black!5}109.68 & \cellcolor{black!5}$2.33\times$ & \cellcolor{black!5}\textbf{23.91} & \cellcolor{black!5}\textbf{0.8279} & \cellcolor{black!5}\underline{0.1073} \\
\cmidrule[0.3pt](lr){2-8}
 & \multirow{5}{*}{$3.0\times$} & TeaCache & 87.71 & $2.92\times$ & 17.39 & 0.6199 & 0.2674 \\
 &  & MagCache & 86.09 & $2.97\times$ & \underline{20.04} & 0.6888 & 0.1910 \\
 &  & DiCache & 86.82 & $2.95\times$ & 19.69 & \underline{0.7170} & \underline{0.1813} \\
 &  & SeaCache & 84.66 & $3.02\times$ & 18.78 & 0.6710 & 0.2226 \\
 &  & \cellcolor{black!5}\textbf{Ours} & \cellcolor{black!5}84.98 & \cellcolor{black!5}$3.01\times$ & \cellcolor{black!5}\textbf{20.95} & \cellcolor{black!5}\textbf{0.7244} & \cellcolor{black!5}\textbf{0.1745} \\
\cmidrule[0.3pt](lr){2-8}
 & \multirow{5}{*}{$4.0\times$} & TeaCache & 64.16 & $3.99\times$ & 14.09 & 0.4625 & 0.4554 \\
 &  & MagCache & 64.09 & $3.99\times$ & 17.03 & 0.4974 & 0.3917 \\
 &  & DiCache & 64.24 & $3.98\times$ & 14.68 & 0.4611 & 0.4198 \\
 &  & SeaCache & 64.91 & $3.94\times$ & \underline{17.13} & \underline{0.5840} & \underline{0.3025} \\
 &  & \cellcolor{black!5}\textbf{Ours} & \cellcolor{black!5}65.49 & \cellcolor{black!5}$3.91\times$ & \cellcolor{black!5}\textbf{18.24} & \cellcolor{black!5}\textbf{0.5891} & \cellcolor{black!5}\textbf{0.2898} \\
\bottomrule[0.7pt]
\end{tabular}
\end{table}

Table~\ref{tab:quality-main} compares MORCA with existing caching methods at comparable speedups. Since MORCA builds on SeaCache by replacing its threshold-based scheduling with a learned scheduler, we focus on this comparison to demonstrate the effectiveness of our learned scheduling policy. On Wan2.2, MORCA achieves the best PSNR, SSIM, and LPIPS at all three target speedups, with PSNR gains over SeaCache of 2.22, 1.86, and 1.78 dB at $1.8\times$, $2.4\times$, and $3.0\times$, respectively. On Wan2.1, it achieves the highest PSNR and SSIM at all three targets, improving PSNR over SeaCache by 1.05, 2.17, and 1.11 dB at $2.4\times$, $3.0\times$, and $4.0\times$, respectively. LPIPS also improves over this baseline at all three settings and is best among all methods at $3.0\times$ and $4.0\times$. Wan2.2 generally exhibits lower fidelity than Wan2.1 at comparable speedups, suggesting greater sensitivity to cache reuse (Sec.~\ref{sec:qualitative-comparison}).

\FloatBarrier %
\subsection{Qualitative Comparison}
\label{sec:qualitative-comparison}

Fig.~\ref{fig:qualitative} presents visual comparisons with existing caching methods. As shown, MORCA better preserves the original structure while retaining fine details, as highlighted in the boxed regions. In particular, we observe that the spatial structure of videos generated by Wan2.2 is more sensitive to cache reuse, leading to relatively lower fidelity metrics. Compared with other methods, MORCA is notably better at preserving the original structure.

\subsection{Ablation And Analysis}
\label{sec:additional-analysis}

\begingroup
\setlength{\parskip}{12.75pt}
\noindent\textbf{Target Speedup Control.}\quad
In Table~\ref{tab:quality-main}, our method achieves mean absolute speedup errors of 0.03 and 0.05, and mean absolute relative inference-time errors of 1.15\% and 1.55\% on Wan2.2 and Wan2.1, respectively. Note that the number of reuse steps is discrete, so attainable acceleration ratios are also discrete, resulting in unavoidable error in speedup control.

\noindent\textbf{Ablation on Latent-Aware State Design.}\quad
Table~\ref{tab:feature-ablation} compares scheduler performance with and without latent features in the state representation. The results show that incorporating latent features enables better scheduling decisions.

\begin{figure}[t]
\centering
\includegraphics[width=\linewidth]{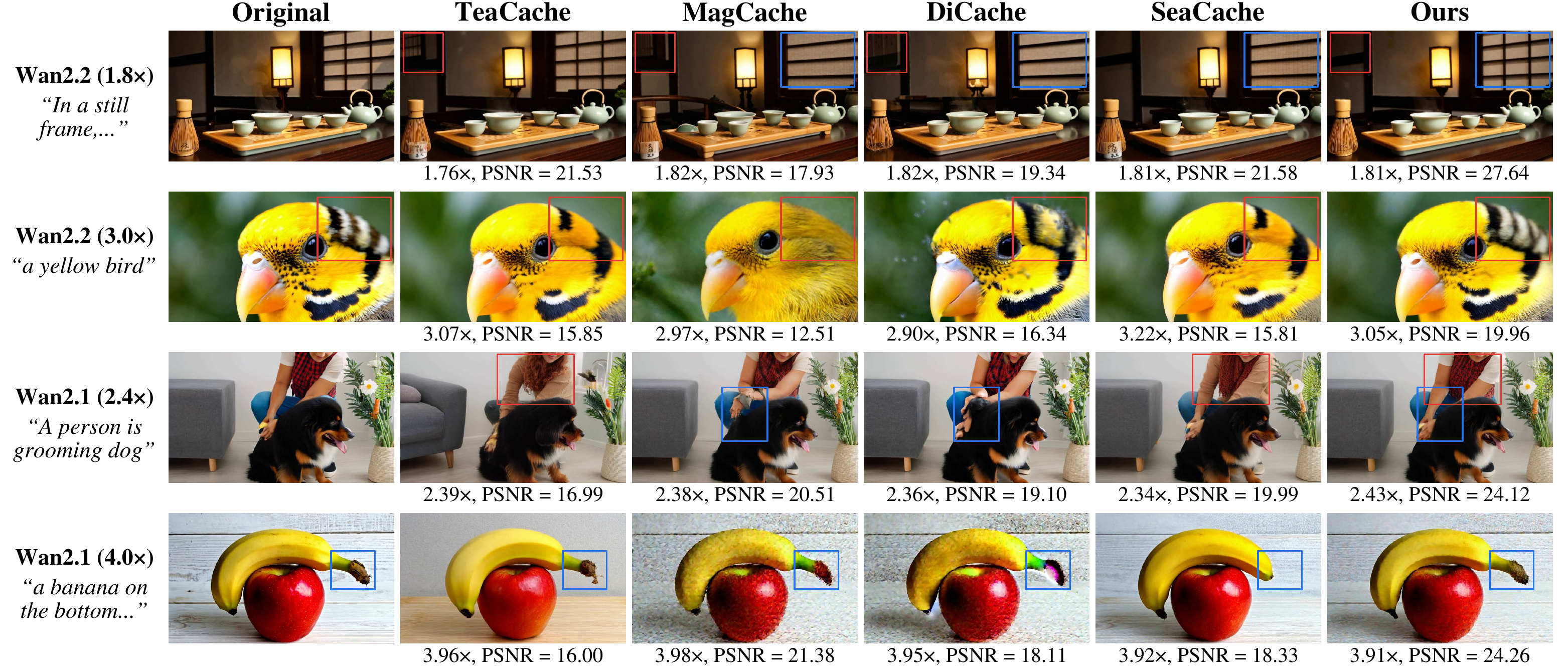}
\caption{\textbf{\fontencoding{T1}\selectfont Qualitative comparison with baselines.} Differences are highlighted in the boxed regions.}
\label{fig:qualitative}
\end{figure}

\begin{table}[t]
\fontencoding{T1}\selectfont
\setlength{\belowcaptionskip}{5pt}
\caption{\textbf{Ablation on Latent-Aware State Design.} Offline training results on Wan2.1 at approximately $3.0\times$ speedup.}
\label{tab:feature-ablation}
\centering
\fontsize{8}{9.2}\selectfont %
\setlength{\tabcolsep}{4pt}
\renewcommand{\arraystretch}{1.1}
\begin{tabular}{lcccccccc}
\toprule[0.7pt]
\textbf{Variant} & $\mathcal{P}_3(x)$ & $\mathcal{P}_2(\mu_T(x))$ & $\mathcal{P}_2(\sigma_T^2(x))$ & \mbox{\textbf{Latency} (s)}$\downarrow$ & \textbf{Speedup}$\uparrow$ & \textbf{PSNR}$\uparrow$ & \textbf{SSIM}$\uparrow$ & \textbf{LPIPS}$\downarrow$ \\
\midrule[0.35pt]
$A_0$ & $\cdot$ & $\cdot$ & $\cdot$ & 85.46 & $2.99\times$ & 19.93 & 0.6420 & 0.2240 \\
$A_1$ & $\checkmark$ & $\cdot$ & $\cdot$ & 86.10 & $2.97\times$ & 19.92 & 0.6515 & 0.2273 \\
$A_2$ & $\cdot$ & $\checkmark$ & $\cdot$ & 86.89 & $2.95\times$ & 19.76 & 0.6429 & 0.2314 \\
$A_3$ & $\cdot$ & $\cdot$ & $\checkmark$ & 85.39 & $3.00\times$ & 19.27 & 0.6151 & 0.2441 \\
$A_4$ & $\checkmark$ & $\checkmark$ & $\cdot$ & 85.38 & $3.00\times$ & 20.06 & 0.6598 & 0.2197 \\
$A_5$ & $\checkmark$ & $\cdot$ & $\checkmark$ & 85.29 & $3.00\times$ & 19.83 & 0.6484 & 0.2320 \\
$A_6$ & $\cdot$ & $\checkmark$ & $\checkmark$ & 85.12 & $3.01\times$ & 19.71 & 0.6402 & 0.2362 \\
Ours & $\checkmark$ & $\checkmark$ & $\checkmark$ & 85.27 & $3.00\times$ & \textbf{20.38} & \textbf{0.6873} & \textbf{0.2036} \\
\bottomrule[0.7pt]
\end{tabular}
\end{table}

\vspace{1pt}
\noindent\begin{minipage}{\linewidth}
\sbox{\trainingtablebox}{%
\fontencoding{T1}\fontsize{8}{9.2}\selectfont%
\setlength{\tabcolsep}{2pt}%
\renewcommand{\arraystretch}{1}%
\begin{tabular}{@{}cccccc@{}}
\toprule[0.7pt]
\textbf{Variant} & \textbf{Latency} (s)$\downarrow$ & \textbf{Speedup}$\uparrow$ & \textbf{PSNR}$\uparrow$ & \textbf{SSIM}$\uparrow$ & \textbf{LPIPS}$\downarrow$ \\
\midrule[0.35pt]
\multicolumn{6}{c}{\textit{(i) IQL vs. REINFORCE}} \\
REINFORCE & 84.66 & $3.02\times$ & 18.97 & 0.6251 & 0.2506 \\
IQL (offline) & 85.27 & $3.00\times$ & \textbf{20.38} & \textbf{0.6873} & \textbf{0.2036} \\
\midrule[0.35pt]
\multicolumn{6}{c}{\textit{(ii) Offline-only vs. online-only}} \\
Online-only & 84.65 & $3.02\times$ & 18.84 & 0.5845 & 0.3212 \\
Offline-only & 85.27 & $3.00\times$ & \textbf{20.38} & \textbf{0.6873} & \textbf{0.2036} \\
\midrule[0.35pt]
\multicolumn{6}{c}{\textit{(iii) Offline-only vs. offline-to-online}} \\
Offline-only & 85.27 & $3.00\times$ & 20.38 & 0.6873 & 0.2036 \\
Offline-to-online & 84.98 & $3.01\times$ & \textbf{20.95} & \textbf{0.7244} & \textbf{0.1745} \\
\bottomrule[0.7pt]
\end{tabular}}
\begin{minipage}[t]{\dimexpr\linewidth-\wd\trainingtablebox-\columnsep\relax}
\vspace{0pt}
\noindent\textbf{Ablation on the Training Paradigm.}\quad
Table~\ref{tab:training-paradigm} examines our training design through three comparisons. With comparable amounts of training data, (i) IQL outperforms REINFORCE, demonstrating the sample efficiency advantage of off-policy learning; and (ii) offline IQL outperforms online IQL trained from scratch, supporting the effectiveness of offline initialization under a limited data budget. Finally, (iii) online fine-tuning further improves the offline-initialized scheduler.
\par
\end{minipage}\hfill
\begin{minipage}[t]{\wd\trainingtablebox}
\vspace{0pt}
\setlength{\abovecaptionskip}{0pt}
\fontencoding{T1}\selectfont
\setlength{\belowcaptionskip}{5pt}
\captionof{table}{\textbf{Ablation on the Training Paradigm.} Results on Wan2.1 at approximately $3.0\times$ speedup.}
\label{tab:training-paradigm}
\centering
\fontsize{8}{9.2}\selectfont %
\setlength{\tabcolsep}{2pt}
\renewcommand{\arraystretch}{1}
\usebox{\trainingtablebox}

\end{minipage}
\end{minipage}\par
\endgroup

\section{Conclusion}

We presented MORCA, a cache scheduling framework for accelerating {video diffusion models}. Motivated by the gap between step error and terminal error that existing methods do not adequately address, along with their limited control over acceleration ratios, we formulated cache scheduling as an MDP, designed a budget-constrained and latent-aware scheduler {for terminal-error-guided decisions under user-specified acceleration targets}, and optimized it through offline-to-online IQL training. Extensive experiments on different video generation models demonstrate that MORCA achieves better {speed--quality trade-offs} than state-of-the-art caching methods. As a schedule-level approach, MORCA can potentially {be combined with different caching granularities and cache reuse strategies for further gains}. Exploring these combinations is a promising direction for future work.

\bibliography{iclr2027_conference}

@InProceedings{peebles2023scalable,
    author    = {Peebles, William and Xie, Saining},
    title     = {Scalable Diffusion Models with Transformers},
    booktitle = {Proceedings of the IEEE/CVF International Conference on Computer Vision (ICCV)},
    month     = {October},
    year      = {2023},
    pages     = {4195-4205}
}

@misc{kong2024hunyuanvideo,
      title={HunyuanVideo: A Systematic Framework For Large Video Generative Models}, 
      author={Weijie Kong and Qi Tian and Zijian Zhang and Rox Min and Zuozhuo Dai and Jin Zhou and Jiangfeng Xiong and Xin Li and Bo Wu and Jianwei Zhang and Kathrina Wu and Qin Lin and Junkun Yuan and Yanxin Long and Aladdin Wang and Andong Wang and Changlin Li and Duojun Huang and Fang Yang and Hao Tan and Hongmei Wang and Jacob Song and Jiawang Bai and Jianbing Wu and Jinbao Xue and Joey Wang and Kai Wang and Mengyang Liu and Pengyu Li and Shuai Li and Weiyan Wang and Wenqing Yu and Xinchi Deng and Yang Li and Yi Chen and Yutao Cui and Yuanbo Peng and Zhentao Yu and Zhiyu He and Zhiyong Xu and Zixiang Zhou and Zunnan Xu and Yangyu Tao and Qinglin Lu and Songtao Liu and Dax Zhou and Hongfa Wang and Yong Yang and Di Wang and Yuhong Liu and Jie Jiang and Caesar Zhong},
      year={2025},
      eprint={2412.03603},
      archivePrefix={arXiv},
      primaryClass={cs.CV},
      url={https://arxiv.org/abs/2412.03603}, 
}

@misc{wan2025wan,
      title={Wan: Open and Advanced Large-Scale Video Generative Models}, 
      author={Team Wan and Ang Wang and Baole Ai and Bin Wen and Chaojie Mao and Chen-Wei Xie and Di Chen and Feiwu Yu and Haiming Zhao and Jianxiao Yang and Jianyuan Zeng and Jiayu Wang and Jingfeng Zhang and Jingren Zhou and Jinkai Wang and Jixuan Chen and Kai Zhu and Kang Zhao and Keyu Yan and Lianghua Huang and Mengyang Feng and Ningyi Zhang and Pandeng Li and Pingyu Wu and Ruihang Chu and Ruili Feng and Shiwei Zhang and Siyang Sun and Tao Fang and Tianxing Wang and Tianyi Gui and Tingyu Weng and Tong Shen and Wei Lin and Wei Wang and Wei Wang and Wenmeng Zhou and Wente Wang and Wenting Shen and Wenyuan Yu and Xianzhong Shi and Xiaoming Huang and Xin Xu and Yan Kou and Yangyu Lv and Yifei Li and Yijing Liu and Yiming Wang and Yingya Zhang and Yitong Huang and Yong Li and You Wu and Yu Liu and Yulin Pan and Yun Zheng and Yuntao Hong and Yupeng Shi and Yutong Feng and Zeyinzi Jiang and Zhen Han and Zhi-Fan Wu and Ziyu Liu},
      year={2025},
      eprint={2503.20314},
      archivePrefix={arXiv},
      primaryClass={cs.CV},
      url={https://arxiv.org/abs/2503.20314}, 
}

@inproceedings{zhao2025realtime,
 author = {Zhao, Xuanlei and Jin, Xiaolong and Wang, Kai and You, Yang},
 booktitle = {International Conference on Learning Representations},
 editor = {Y. Yue and A. Garg and N. Peng and F. Sha and R. Yu},
 pages = {3296--3319},
 title = {Real-Time Video Generation with Pyramid Attention Broadcast},
 url = {https://proceedings.iclr.cc/paper_files/paper/2025/file/092c2d45005ea2db40fc24c470663416-Paper-Conference.pdf},
 volume = {2025},
 year = {2025}
}

@InProceedings{liu2025teacache,
    author    = {Liu, Feng and Zhang, Shiwei and Wang, Xiaofeng and Wei, Yujie and Qiu, Haonan and Zhao, Yuzhong and Zhang, Yingya and Ye, Qixiang and Wan, Fang},
    title     = {Timestep Embedding Tells: It's Time to Cache for Video Diffusion Model},
    booktitle = {Proceedings of the IEEE/CVF Conference on Computer Vision and Pattern Recognition (CVPR)},
    month     = {June},
    year      = {2025},
    pages     = {7353-7363}
}

@InProceedings{kahatapitiya2025adacache,
    author    = {Kahatapitiya, Kumara and Liu, Haozhe and He, Sen and Liu, Ding and Jia, Menglin and Zhang, Chenyang and Ryoo, Michael S. and Xie, Tian},
    title     = {Adaptive Caching for Faster Video Generation with Diffusion Transformers},
    booktitle = {Proceedings of the IEEE/CVF International Conference on Computer Vision (ICCV)},
    month     = {October},
    year      = {2025},
    pages     = {15240-15252}
}

@InProceedings{chung2026seacache,
    author    = {Chung, Jiwoo and Hyun, Sangeek and Lee, MinKyu and Han, Byeongju and Cha, Geonho and Wee, Dongyoon and Hong, Youngjun and Heo, Jae-Pil},
    title     = {SeaCache: Spectral-Evolution-Aware Cache for Accelerating Diffusion Models},
    booktitle = {Proceedings of the IEEE/CVF Conference on Computer Vision and Pattern Recognition (CVPR)},
    month     = {June},
    year      = {2026},
    pages     = {14283-14294}
}

@inproceedings{lipman2023flow,
  title = {Flow Matching for Generative Modeling},
  author = {Yaron Lipman and Ricky T. Q. Chen and Heli Ben-Hamu and Maximilian Nickel and Matthew Le},
  booktitle = {International Conference on Learning Representations},
  year = {2023},
  url = {https://openreview.net/forum?id=PqvMRDCJT9t}
}

@inproceedings{liu2023flow,
  title = {Flow Straight and Fast: Learning to Generate and Transfer Data with Rectified Flow},
  author = {Xingchao Liu and Chengyue Gong and Qiang Liu},
  booktitle = {International Conference on Learning Representations},
  year = {2023},
  url = {https://openreview.net/forum?id=XVjTT1nw5z}
}

@InProceedings{ma2024deepcache,
    author    = {Ma, Xinyin and Fang, Gongfan and Wang, Xinchao},
    title     = {DeepCache: Accelerating Diffusion Models for Free},
    booktitle = {Proceedings of the IEEE/CVF Conference on Computer Vision and Pattern Recognition (CVPR)},
    month     = {June},
    year      = {2024},
    pages     = {15762-15772}
}

@inproceedings{li2024fasterdiffusion,
 author = {Li, Senmao and Hu, Taihang and van de Weijer, Joost and Khan, Fahad Shahbaz and Liu, Tao and Li, Linxuan and Yang, Shiqi and Wang, Yaxing and Cheng, Ming-Ming and Yang, Jian},
 booktitle = {Advances in Neural Information Processing Systems},
 doi = {10.52202/079017-2706},
 editor = {A. Globerson and L. Mackey and D. Belgrave and A. Fan and U. Paquet and J. Tomczak and C. Zhang},
 pages = {85203--85240},
 publisher = {Curran Associates, Inc.},
 title = {Faster Diffusion: Rethinking the Role of the Encoder for Diffusion Model Inference},
 url = {https://proceedings.neurips.cc/paper_files/paper/2024/file/9ad996b5c45130de2bc00b60d8607904-Paper-Conference.pdf},
 volume = {37},
 year = {2024}
}

@inproceedings{lv2024fastercache,
 author = {Lyu, Zhengyao and Si, Chenyang and Song, Junhao and Yang, Zhenyu and Qiao, Yu and Liu, Ziwei and Wong, Kwan-Yee K},
 booktitle = {International Conference on Learning Representations},
 editor = {Y. Yue and A. Garg and N. Peng and F. Sha and R. Yu},
 pages = {33132--33156},
 title = {FasterCache: Training-Free Video Diffusion Model Acceleration with High Quality},
 url = {https://proceedings.iclr.cc/paper_files/paper/2025/file/518046d86bbc41a0707727c38301ad8e-Paper-Conference.pdf},
 volume = {2025},
 year = {2025}
}

@InProceedings{liu2025taylorseer,
    author    = {Liu, Jiacheng and Zou, Chang and Lyu, Yuanhuiyi and Chen, Junjie and Zhang, Linfeng},
    title     = {From Reusing to Forecasting: Accelerating Diffusion Models with TaylorSeers},
    booktitle = {Proceedings of the IEEE/CVF International Conference on Computer Vision (ICCV)},
    month     = {October},
    year      = {2025},
    pages     = {15853-15863}
}

@inproceedings{ma2025magcache,
 author = {Ma, Zehong and Wei, Longhui and Wang, Feng and Zhang, Shiliang and Tian, Qi},
 booktitle = {Advances in Neural Information Processing Systems},
 doi = {10.52202/085713-1153},
 editor = {D. Belgrave and C. Zhang and H. Lin and R. Pascanu and P. Koniusz and M. Ghassemi and N. Chen},
 pages = {34348--34380},
 publisher = {Curran Associates, Inc.},
 title = {MagCache: Fast Video Generation with Magnitude-Aware Cache},
 url = {https://proceedings.neurips.cc/paper_files/paper/2025/file/311207bb626e36a8f1d3eb92aa67af22-Paper-Conference.pdf},
 volume = {38, Main Conference},
 year = {2025}
}

@inproceedings{bu2026dicache,
 author = {Bu, Jiazi and Ling, Pengyang and Zhou, Yujie and Wang, Yibin and Zang, Yuhang and Lin, Dahua and Wang, Jiaqi},
 booktitle = {International Conference on Learning Representations},
 editor = {C. Vondrick and B. Hariharan and C. Raffel and L. Pinto and D. Yang and A. Faust},
 pages = {73778--73803},
 title = {DiCache: Let Diffusion Model Determine Its Own Cache},
 url = {https://proceedings.iclr.cc/paper_files/paper/2026/file/78288ef33b18a351c3cd679dc9a15c8d-Paper-Conference.pdf},
 volume = {2026},
 year = {2026}
}

@InProceedings{wimbauer2024cacheme,
    author    = {Wimbauer, Felix and Wu, Bichen and Schoenfeld, Edgar and Dai, Xiaoliang and Hou, Ji and He, Zijian and Sanakoyeu, Artsiom and Zhang, Peizhao and Tsai, Sam and Kohler, Jonas and Rupprecht, Christian and Cremers, Daniel and Vajda, Peter and Wang, Jialiang},
    title     = {Cache Me if You Can: Accelerating Diffusion Models through Block Caching},
    booktitle = {Proceedings of the IEEE/CVF Conference on Computer Vision and Pattern Recognition (CVPR)},
    month     = {June},
    year      = {2024},
    pages     = {6211-6220}
}

@misc{chen2024deltadit,
      title={$\Delta$-DiT: A Training-Free Acceleration Method Tailored for Diffusion Transformers}, 
      author={Pengtao Chen and Mingzhu Shen and Peng Ye and Jianjian Cao and Chongjun Tu and Christos-Savvas Bouganis and Yiren Zhao and Tao Chen},
      year={2024},
      eprint={2406.01125},
      archivePrefix={arXiv},
      primaryClass={cs.CV},
      url={https://arxiv.org/abs/2406.01125}, 
}

@inproceedings{zou2025toca,
 author = {Zou, Chang and Liu, Xuyang and Liu, Ting and Huang, Siteng and Zhang, Linfeng},
 booktitle = {International Conference on Learning Representations},
 editor = {Y. Yue and A. Garg and N. Peng and F. Sha and R. Yu},
 pages = {75467--75488},
 title = {Accelerating Diffusion Transformers with Token-wise Feature Caching},
 url = {https://proceedings.iclr.cc/paper_files/paper/2025/file/bbe024e0517fe12ac3a8d388b19ff9fe-Paper-Conference.pdf},
 volume = {2025},
 year = {2025}
}

@inproceedings{zhao2026rapid3,
 author = {Zhao, Wangbo and Han, Yizeng and Tang, Zhiwei and Tang, Jiasheng and Zhou, Pengfei and Wang, Kai and Zhuang, Bohan and Wang, Zhangyang and Wang, Fan and You, Yang},
 booktitle = {International Conference on Learning Representations},
 editor = {C. Vondrick and B. Hariharan and C. Raffel and L. Pinto and D. Yang and A. Faust},
 pages = {149423--149446},
 title = {RAPID\^{}3: Tri-Level Reinforced Acceleration Policies for Diffusion Transformer},
 url = {https://proceedings.iclr.cc/paper_files/paper/2026/file/f1cf02ce09757f57c3b93c0db83181e0-Paper-Conference.pdf},
 volume = {2026},
 year = {2026}
}

@misc{xie2026onlinecache,
      title={OnlineCache: Learning Dynamic Caching Policies with Error Correction for Efficient Diffusion Inference}, 
      author={Zhikang Xie and Xichen Ye and Yifan Wu and Haoshen Yu and Li chenan and Peizhu Gong and Weizhong Zhang and Cheng Jin},
      year={2026},
      eprint={2607.29398},
      archivePrefix={arXiv},
      primaryClass={cs.LG},
      url={https://arxiv.org/abs/2607.29398}, 
}

@misc{zhao2026bag,
      title={BAG: Budget-Aware Gating for Diffusion Caching}, 
      author={Tong Zhao and Mingkun Lei and Yucheng Han and Chi Zhang},
      year={2026},
      eprint={2608.09231},
      archivePrefix={arXiv},
      primaryClass={cs.CV},
      url={https://arxiv.org/abs/2608.09231}, 
}

@InProceedings{mcmahan2024anytime,
  title = 	 {Anytime-Constrained Reinforcement Learning},
  author =       {McMahan, Jeremy and Zhu, Xiaojin},
  booktitle = 	 {Proceedings of The 27th International Conference on Artificial Intelligence and Statistics},
  pages = 	 {4321--4329},
  year = 	 {2024},
  editor = 	 {Dasgupta, Sanjoy and Mandt, Stephan and Li, Yingzhen},
  volume = 	 {238},
  series = 	 {Proceedings of Machine Learning Research},
  month = 	 {02--04 May},
  publisher =    {PMLR},
  url = 	 {https://proceedings.mlr.press/v238/mcmahan24a.html}
}

@inproceedings{kostrikov2022offline,
  title = {Offline Reinforcement Learning with Implicit Q-Learning},
  author = {Ilya Kostrikov and Ashvin Nair and Sergey Levine},
  booktitle = {International Conference on Learning Representations},
  year = {2022},
  url = {https://openreview.net/forum?id=68n2s9ZJWF8}
}

@InProceedings{huang2024vbench,
    author    = {Huang, Ziqi and He, Yinan and Yu, Jiashuo and Zhang, Fan and Si, Chenyang and Jiang, Yuming and Zhang, Yuanhan and Wu, Tianxing and Jin, Qingyang and Chanpaisit, Nattapol and Wang, Yaohui and Chen, Xinyuan and Wang, Limin and Lin, Dahua and Qiao, Yu and Liu, Ziwei},
    title     = {VBench: Comprehensive Benchmark Suite for Video Generative Models},
    booktitle = {Proceedings of the IEEE/CVF Conference on Computer Vision and Pattern Recognition (CVPR)},
    month     = {June},
    year      = {2024},
    pages     = {21807-21818}
}

@inproceedings{
song2021denoising,
title={Denoising Diffusion Implicit Models},
author={Jiaming Song and Chenlin Meng and Stefano Ermon},
booktitle={International Conference on Learning Representations},
year={2021},
url={https://openreview.net/forum?id=St1giarCHLP}
}

@inproceedings{liu2022pseudo,
    title={Pseudo Numerical Methods for Diffusion Models on Manifolds},
    author={Luping Liu and Yi Ren and Zhijie Lin and Zhou Zhao},
    booktitle={International Conference on Learning Representations},
    year={2022},
    url={https://openreview.net/forum?id=PlKWVd2yBkY}
}

@inproceedings{NEURIPS2022_260a14ac,
 author = {Lu, Cheng and Zhou, Yuhao and Bao, Fan and Chen, Jianfei and LI, Chongxuan and Zhu, Jun},
 booktitle = {Advances in Neural Information Processing Systems},
 doi = {10.52202/068431-0418},
 editor = {S. Koyejo and S. Mohamed and A. Agarwal and D. Belgrave and K. Cho and A. Oh},
 pages = {5775--5787},
 publisher = {Curran Associates, Inc.},
 title = {DPM-Solver: A Fast ODE Solver for Diffusion Probabilistic Model Sampling in Around 10 Steps},
 url = {https://proceedings.neurips.cc/paper_files/paper/2022/file/260a14acce2a89dad36adc8eefe7c59e-Paper-Conference.pdf},
 volume = {35},
 year = {2022}
}

@Article{Lu2025,
author="Lu, Cheng
and Zhou, Yuhao
and Bao, Fan
and Chen, Jianfei
and Li, Chongxuan
and Zhu, Jun",
title="DPM-Solver++: Fast Solver for Guided Sampling of Diffusion Probabilistic Models",
journal="Machine Intelligence Research",
year="2025",
month="Aug",
day="01",
volume="22",
number="4",
pages="730--751",
issn="2731-5398",
doi="10.1007/s11633-025-1562-4",
url="https://doi.org/10.1007/s11633-025-1562-4"
}

@inproceedings{NEURIPS2023_9c2aa1e4,
 author = {Zhao, Wenliang and Bai, Lujia and Rao, Yongming and Zhou, Jie and Lu, Jiwen},
 booktitle = {Advances in Neural Information Processing Systems},
 doi = {10.52202/075280-2170},
 editor = {A. Oh and T. Naumann and A. Globerson and K. Saenko and M. Hardt and S. Levine},
 pages = {49842--49869},
 publisher = {Curran Associates, Inc.},
 title = {UniPC: A Unified Predictor-Corrector Framework for Fast Sampling of Diffusion Models},
 url = {https://proceedings.neurips.cc/paper_files/paper/2023/file/9c2aa1e456ea543997f6927295196381-Paper-Conference.pdf},
 volume = {36},
 year = {2023}
}

@inproceedings{
salimans2022progressive,
title={Progressive Distillation for Fast Sampling of Diffusion Models},
author={Tim Salimans and Jonathan Ho},
booktitle={International Conference on Learning Representations},
year={2022},
url={https://openreview.net/forum?id=TIdIXIpzhoI}
}

@InProceedings{pmlr-v202-song23a,
  title = 	 {Consistency Models},
  author =       {Song, Yang and Dhariwal, Prafulla and Chen, Mark and Sutskever, Ilya},
  booktitle = 	 {Proceedings of the 40th International Conference on Machine Learning},
  pages = 	 {32211--32252},
  year = 	 {2023},
  editor = 	 {Krause, Andreas and Brunskill, Emma and Cho, Kyunghyun and Engelhardt, Barbara and Sabato, Sivan and Scarlett, Jonathan},
  volume = 	 {202},
  series = 	 {Proceedings of Machine Learning Research},
  month = 	 {23--29 Jul},
  publisher =    {PMLR},
  url = 	 {https://proceedings.mlr.press/v202/song23a.html}
}

@misc{luo2023latentconsistencymodelssynthesizing,
      title={Latent Consistency Models: Synthesizing High-Resolution Images with Few-Step Inference}, 
      author={Simian Luo and Yiqin Tan and Longbo Huang and Jian Li and Hang Zhao},
      year={2023},
      eprint={2310.04378},
      archivePrefix={arXiv},
      primaryClass={cs.CV},
      url={https://arxiv.org/abs/2310.04378}, 
}

@InProceedings{Yin_2024_CVPR,
    author    = {Yin, Tianwei and Gharbi, Micha\"el and Zhang, Richard and Shechtman, Eli and Durand, Fr\'edo and Freeman, William T. and Park, Taesung},
    title     = {One-step Diffusion with Distribution Matching Distillation},
    booktitle = {Proceedings of the IEEE/CVF Conference on Computer Vision and Pattern Recognition (CVPR)},
    month     = {June},
    year      = {2024},
    pages     = {6613-6623}
}

@InProceedings{Shang_2023_CVPR,
    author    = {Shang, Yuzhang and Yuan, Zhihang and Xie, Bin and Wu, Bingzhe and Yan, Yan},
    title     = {Post-Training Quantization on Diffusion Models},
    booktitle = {Proceedings of the IEEE/CVF Conference on Computer Vision and Pattern Recognition (CVPR)},
    month     = {June},
    year      = {2023},
    pages     = {1972-1981}
}

@InProceedings{Li_2023_ICCV,
    author    = {Li, Xiuyu and Liu, Yijiang and Lian, Long and Yang, Huanrui and Dong, Zhen and Kang, Daniel and Zhang, Shanghang and Keutzer, Kurt},
    title     = {Q-Diffusion: Quantizing Diffusion Models},
    booktitle = {Proceedings of the IEEE/CVF International Conference on Computer Vision (ICCV)},
    month     = {October},
    year      = {2023},
    pages     = {17535-17545}
}

@inproceedings{NEURIPS2023_2aab8a76,
 author = {He, Yefei and Liu, Luping and Liu, Jing and Wu, Weijia and Zhou, Hong and Zhuang, Bohan},
 booktitle = {Advances in Neural Information Processing Systems},
 doi = {10.52202/075280-0580},
 editor = {A. Oh and T. Naumann and A. Globerson and K. Saenko and M. Hardt and S. Levine},
 pages = {13237--13249},
 publisher = {Curran Associates, Inc.},
 title = {PTQD: Accurate Post-Training Quantization for Diffusion Models},
 url = {https://proceedings.neurips.cc/paper_files/paper/2023/file/2aab8a76c7e761b66eccaca0927787de-Paper-Conference.pdf},
 volume = {36},
 year = {2023}
}

@inproceedings{ICLR2025_a4a1ee07,
 author = {Zhao, Tianchen and Fang, Tongcheng and Huang, Haofeng and Wan, Rui and Soedarmadji, Widyadewi and Liu, Enshu and Li, Shiyao and Lin, Zinan and Dai, Guohao and Yan, Shengen and Yang, Huazhong and Ning, Xuefei and Wang, Yu},
 booktitle = {International Conference on Learning Representations},
 editor = {Y. Yue and A. Garg and N. Peng and F. Sha and R. Yu},
 pages = {65811--65841},
 title = {ViDiT-Q: Efficient and Accurate Quantization of Diffusion Transformers for Image and Video Generation},
 url = {https://proceedings.iclr.cc/paper_files/paper/2025/file/a4a1ee071ce0fe63b83bce507c9dc4d7-Paper-Conference.pdf},
 volume = {2025},
 year = {2025}
}

@InProceedings{pmlr-v267-xi25c,
  title = 	 {Sparse Video-Gen: Accelerating Video Diffusion Transformers with Spatial-Temporal Sparsity},
  author =       {Xi, Haocheng and Yang, Shuo and Zhao, Yilong and Xu, Chenfeng and Li, Muyang and Li, Xiuyu and Lin, Yujun and Cai, Han and Zhang, Jintao and Li, Dacheng and Chen, Jianfei and Stoica, Ion and Keutzer, Kurt and Han, Song},
  booktitle = 	 {Proceedings of the 42nd International Conference on Machine Learning},
  pages = 	 {68208--68224},
  year = 	 {2025},
  editor = 	 {Singh, Aarti and Fazel, Maryam and Hsu, Daniel and Lacoste-Julien, Simon and Berkenkamp, Felix and Maharaj, Tegan and Wagstaff, Kiri and Zhu, Jerry},
  volume = 	 {267},
  series = 	 {Proceedings of Machine Learning Research},
  month = 	 {13--19 Jul},
  publisher =    {PMLR},
  url = 	 {https://proceedings.mlr.press/v267/xi25c.html}
}
\bibliographystyle{iclr2027_conference}

\clearpage
\appendix
\section*{Appendix}
\begingroup
\hypersetup{hidelinks}
\newcommand{\appendixentry}[3]{%
  \par\noindent\hspace*{#1em}%
  \hyperref[#2]{\ref*{#2}\quad #3}\nobreak\dotfill\pageref{#2}\par\smallskip}
\appendixentry{0}{app:model-configurations}{Model Configurations}
\appendixentry{0}{app:training-details}{MORCA Implementation Details}
\appendixentry{1}{app:network-architecture}{Network Architecture}
\appendixentry{1}{app:training-hyperparameters}{Training Hyperparameters}
\appendixentry{1}{app:dataset-preparation}{Dataset Preparation}
\appendixentry{0}{app:baseline-details}{Baseline Implementation Details}
\appendixentry{0}{app:budget-proof}{Reduction of the Hard-Constrained MDP}
\appendixentry{0}{app:additional-results}{Additional Experiments}
\appendixentry{1}{app:qualitative-results}{Additional Qualitative Comparisons}
\appendixentry{1}{app:scheduler-overhead}{Scheduler Overhead}
\appendixentry{1}{app:seacache-comparison}{Detailed Comparisons with SeaCache}
\appendixentry{2}{app:matched-budget-schedules}{Schedule Visualizations at the Same Reuse Budget}
\appendixentry{2}{app:seacache-speedups}{Additional Target Acceleration Ratios}
\appendixentry{0}{app:limitations}{Limitations}
\endgroup
\clearpage

\section{Model Configurations}
\label{app:model-configurations}

Table~\ref{tab:model-configurations} summarizes the inference configurations for Wan2.2 and Wan2.1. Both models use 50 denoising steps at a resolution of $832\times480$ and a frame rate of 16 fps. Wan2.2 generates 45-frame videos using DPM++, while Wan2.1 generates 81-frame videos using UniPC. Model-specific settings, including flow shift, CFG scale, and model offloading, are also reported.

\begin{table}[htbp]
\fontencoding{T1}\selectfont
\centering
\setlength{\belowcaptionskip}{5pt}
\caption{{\textbf{Inference configurations for Wan2.2 and Wan2.1.}}}
\label{tab:model-configurations}
\fontsize{8}{9.2}\selectfont
\begin{tabular}{lcc}
\hline
\textbf{Setting} & \textbf{{Wan2.2}} & \textbf{{Wan2.1}} \\
\hline
Resolution ($W\times H$) & $832\times480$ & $832\times480$ \\
Frames & 45 & 81 \\
Frame rate & 16 fps & 16 fps \\
Sampler & DPM++ & UniPC \\
Denoising steps & 50 & 50 \\
Flow shift & 12 & 5 \\
CFG scale & $4/3$ (high/low) & 5 \\
Noise-stage boundary & 0.875 & N/A \\
Model offloading & Enabled & Disabled \\
\hline
\end{tabular}
\end{table}

\section{MORCA Implementation Details}
\label{app:training-details}

\subsection{Network Architecture}
\label{app:network-architecture}

For both video generation models, {we use four separate networks} for $\pi_\varphi$, $Q_{\theta_1}$, $Q_{\theta_2}$, and $V_\psi$. {The pooled features of $x_t$, $x_{t+1}$, and $x_{j_t}$ form a $48\times4\times8\times8$ spatiotemporal tensor and two $48\times8\times8$ tensors representing temporal means and variances, respectively (batch dimension omitted). The spatiotemporal tensor passes through a 3D CNN branch, while the temporal mean and variance tensors are concatenated along the channel dimension and passed through a 2D CNN branch. Each branch has three convolutional layers of widths $(32,32,64)$, followed by adaptive average pooling and flattening, producing 512- and 256-dimensional vectors for the 3D and 2D branches, respectively. Concatenation and a linear projection followed by SiLU map the combined 768-dimensional vector to a 128-dimensional latent embedding. Appending seven scalar state features---$d_t$, $D_t$, $t$, $\ell_t$, $u_t$, $K$, and a cache-valid bit---gives a 135-dimensional input to an MLP with two 256-unit hidden layers, each followed by LayerNorm and SiLU.} The policy outputs two action logits, each Q-network outputs two action values, and the V-network outputs a scalar. Two additional target Q-networks share the architecture of their corresponding Q-networks and are updated through EMA. {Only the policy network is needed at inference time. Further implementation details will be provided in the code release.}

\subsection{Training Hyperparameters}
\label{app:training-hyperparameters}

Table~\ref{tab:iql-hyperparameters} summarizes the training and IQL settings, where $\gamma$ denotes the discount factor, $\tau$ the expectile level for value regression, and $\beta$ the advantage-weighting coefficient.
{The weight cap $w_{\max}$ limits the policy weights as $\min(\exp(\beta\widetilde{A}),w_{\max})$, where $\widetilde{A}$ is the normalized advantage. The EMA decay $\eta$ controls target-Q network updates via $\theta_{\mathrm{target}}\leftarrow\eta\theta_{\mathrm{target}}+(1-\eta)\theta_Q$.}

\begin{table}[htbp]
\fontencoding{T1}\selectfont
\centering
\setlength{\belowcaptionskip}{5pt}
\caption{{\textbf{Training and IQL hyperparameters for the two stages, shared across Wan2.2 and Wan2.1.}}}
\label{tab:iql-hyperparameters}
\fontsize{8}{9.2}\selectfont
\setlength{\tabcolsep}{4pt}
\begin{tabular}{llcc}
\hline
\textbf{Group} & \textbf{Setting} & \textbf{Offline IQL} & \textbf{Online fine-tuning} \\
\hline
\multirow{9}{*}{Training} & Training duration & {200 epochs} & {16 rounds} \\
 & Optimizer & AdamW & AdamW \\
 & $\pi$ LR & $10^{-4}$ & $4\times10^{-5}$ \\
 & Q/V LR & $10^{-4}$ & $10^{-4}$ \\
 & Batch size & 256 & 256 \\
 & Adam betas & $(0.9,0.999)$ & $(0.9,0.999)$ \\
 & Adam epsilon & $10^{-8}$ & $10^{-8}$ \\
 & Weight decay & 0.01 & 0.01 \\
 & Maximum gradient norm & 1.0 & 1.0 \\
\hline
\multirow{8}{*}{IQL} & $\gamma$ & 1.0 & 1.0 \\
 & $\tau$ & {0.7} & 0.7 \\
 & $\beta$ & {1.5} & {1.0} \\
 & $w_{\max}$ & {30} & {20} \\
 & $\eta$ & 0.999 & 0.999 \\
 & New trajectories & N/A & 100 per round \\
 & Critic-only warm-up & N/A & 5 epochs per round \\
 & Joint updates & N/A & 20 epochs per round \\
\hline
\end{tabular}
\end{table}

\subsection{Dataset Preparation}
\label{app:dataset-preparation}

\paragraph{Data Generation Pipeline.}
\label{app:dataset-construction}

A straightforward way to generate random cache schedules is to independently sample a reuse/recompute decision at each denoising step. However, this is likely to produce trajectories with limited learning value, wasting costly generation rollouts. Instead, we adopt a more stable approach by randomizing at the threshold level: we sample a smooth threshold curve that varies across the $N$ denoising steps, replace SeaCache's fixed threshold with this curve, and run inference to collect the resulting denoising trajectory.

To generate a random threshold curve, we control three aspects: its mean level, amplitude, and shape. The mean level is determined by mapping a sampled target acceleration ratio ${\rho^\star}$ to a threshold center $\mu=g({\rho^\star})$ through calibration. A randomly sampled amplitude factor determines the magnitude of threshold variations. To generate the shape, we randomly select a few reference steps from the $N$ denoising steps, sample random relative threshold values at these steps, and apply cubic smoothstep interpolation across all steps. These three factors jointly determine the threshold curve, followed by adjustments to satisfy the required bounds and smoothness constraints.

{After each rollout, we convert the trajectory into $N$ transition tuples. For each step $t$, we compute the cumulative reuse count $u_t=\sum_{i=t+1}^{N}a_i$ and assign the trajectory's total reuse count $K=\sum_{i=1}^{N}a_i$ as the reuse budget. We combine these counts with the step-error signals, step statistics, and latent features extracted from the recorded denoising process to construct the state vector $z_t$. Using the reward defined in the main text, each trajectory yields $N$ transition tuples:}
\[
{(z_t,a_t,r_t,z_{t-1}),\qquad t=N,\ldots,1.}
\]
Further implementation details will be provided in the code release.

\paragraph{Dataset scale.}

We collect 2,500 random cache-scheduling trajectories for Wan2.2 (100 prompts $\times$ 25 schedules), covering speedups from $1.5\times$ to $3.5\times$, and {4,000} for Wan2.1 (1,000 prompts $\times$ {4} schedules){, covering speedups from $2.0\times$ to $4.5\times$}.

\paragraph{Forced-action handling.}

We identify two types of forced actions: (1) actions imposed by the exact reuse-budget constraint, when only one action remains feasible; and (2) mandatory recomputation at the first {denoising step and the last denoising step}, as well as the high-noise/low-noise expert-switching step in Wan2.2. Since these actions are determined by action restrictions rather than policy choices, we exclude the corresponding transitions from the policy loss while retaining them for the standard Q- and V-function updates.

\section{Baseline Implementation Details}
\label{app:baseline-details}

For Wan2.2, we port the official Wan2.1 implementations and add mandatory recomputation at the high-noise/low-noise expert-switching step to initialize the newly activated expert's cache. MagCache uses its official Wan2.2 implementation directly. For TeaCache on Wan2.2, we fit separate fourth-degree polynomials for the high- and low-noise stages using 70 calibration prompts. For Wan2.1, we use the official Wan2.1 implementations released by the respective caching methods. Table~\ref{tab:baseline-configurations} summarizes the key configurations for three approximate acceleration regimes.

\paragraph{Method-specific settings.}
TeaCache and SeaCache use \texttt{use\_ret\_steps=False} for both models. For Wan2.1, SeaCache additionally uses \texttt{power\_exp=3} and \texttt{power\_const=1}. DiCache uses \texttt{probe\_depth=1} for both models, with DCTA enabled and its correction coefficients clipped to $[1,2]$.

\begin{table}[htbp]
\fontencoding{T1}\selectfont
\centering
\setlength{\belowcaptionskip}{5pt}
\caption{{\textbf{Baseline configurations at three target acceleration ratios for each model.}} $\delta$ denotes the cache threshold, $K$ the maximum number of consecutive reuse steps for MagCache, and $R$ the retention ratio.}
\label{tab:baseline-configurations}
\fontsize{8}{9.2}\selectfont
\setlength{\tabcolsep}{5pt}
\thinmuskip=2mu \medmuskip=2mu \thickmuskip=2mu
\renewcommand{\arraystretch}{1.2}
\begin{tabular}{llccc}
\hline
\textbf{Model} & \textbf{Method} & {$\boldsymbol{1.8\times}$} & {$\boldsymbol{2.4\times}$} & {$\boldsymbol{3.0\times}$} \\
\hline
\multirow[c]{4}{*}{Wan2.2} & TeaCache & $\delta=0.29$ & $\delta=0.45$ & $\delta=0.68$ \\
 & MagCache & $\delta=0.072$, $K=2$, $R=0.2$ & $\delta=0.199$, $K=4$, $R=0.2$ & $\delta=0.198$, $K=5$, $R=0.1$ \\
 & DiCache & $\delta=0.075$, {$R=0.2$} & $\delta=0.172$, {$R=0.2$} & $\delta=0.396$, {$R=0.2$} \\
 & SeaCache & $\delta=0.24$ & $\delta=0.38$ & $\delta=0.55$ \\
\hline
\textbf{Model} & \textbf{Method} & {$\boldsymbol{2.4\times}$} & {$\boldsymbol{3.0\times}$} & {$\boldsymbol{4.0\times}$} \\
\hline
\multirow[c]{4}{*}{Wan2.1} & TeaCache & $\delta=0.142$ & $\delta=0.195$ & {$\delta=0.36$} \\
 & MagCache & $\delta=0.1500$, $K=6$, {$R=0.2$} & $\delta=0.9246$, $K=8$, {$R=0.2$} & {$\delta=8.0$}, {$K=40$}, {$R=0.2$} \\
 & DiCache & $\delta=0.155$, {$R=0.2$} & $\delta=0.365$, {$R=0.2$} & {$\delta=12.0$}, {$R=0.17$} \\
 & SeaCache & $\delta=0.28$ & $\delta=0.47$ & {$\delta=0.7$} \\
\hline
\end{tabular}
\end{table}

\FloatBarrier

\section{Reduction of the Hard-Constrained MDP}
\label{app:budget-proof}

We use the MDP and terminal reward defined in the main text, with $s_t$ denoting the full Markov state and $P_t$ {denoting the step-$t$ transition kernel corresponding to $P$ in the main text}. Assume standard Borel state spaces and measurable transition and reward kernels. PSNR is capped at {the fixed finite maximum $r_{\max}$ defined in the main text}, yielding bounded rewards for videos in a fixed bounded pixel range.

The displayed formulas assume that every step permits reuse. Fix an integer budget $K\in\{0,\ldots,N\}$. {Let $b_t=K-u_t$ denote the remaining reuse budget, as in the main text.} Recall the augmented state and action restriction:
\begin{equation}
\begin{aligned}
\bar s_t&=(s_t,{u_t,K}),\qquad
u_t=\sum_{i=t+1}^{N}a_i,\\
\bar{\mathcal A}_t(\bar s_t)
&=\{a\in\{0,1\}:0\le {b_t}-a\le t-1\}.
\end{aligned}
\label{eq:appendix-action-restriction}
\end{equation}
Only integer counts satisfying $0\le u_t\le\min\{K,N-t\}$ and $0\le {b_t}\le t$ are included. The augmented process preserves the initial distribution and transitions $P_t$ of the original state component and all rewards, with $u_N=0$ and $u_{t-1}=u_t+a_t$. Predetermined recomputation steps are handled by excluding reuse at those steps and adjusting the budget range, bounds on the cumulative reuse count, and action restriction to count only steps at which reuse is permitted.

\begin{proposition}
\label{prop:hard-budget-reduction}
Let $\Pi_K^{\mathrm{hist}}$ denote all possibly randomized history-dependent policies satisfying $\sum_{t=1}^{N}a_t=K$ almost surely. Let $\bar\Pi^{\mathrm{hist}}$ and $\bar\Pi^{\mathrm{det}}$ denote, respectively, all admissible possibly randomized history-dependent policies and all admissible time-dependent deterministic Markov policies of the augmented MDP. The reduction guarantees exactly $K$ reuse actions almost surely and preserves the optimal feasible value:
\begin{equation}
\sup_{\pi\in\Pi_K^{\mathrm{hist}}}
\mathbb E_{\tau\sim\pi}[R(\tau)]
=
\max_{\bar\pi\in\bar\Pi^{\mathrm{det}}}
\mathbb E_{\bar\tau\sim\bar\pi}[R(\bar\tau)],
\label{eq:appendix-optimal-feasible-value}
\end{equation}
where $R(\bar\tau)$ denotes the unchanged return along the corresponding original trajectory.

\end{proposition}

\begin{proof}
The augmented process is Markov because $s_t$ is Markov and the count update is deterministic.

\noindent\textbf{Step 1. Feasibility.}
It is straightforward to verify that the augmented action sets are nonempty at every nonterminal state, every admissible augmented policy satisfies the prescribed reuse budget almost surely, and every feasible original policy respects the augmented action restriction almost surely.

\noindent\textbf{Step 2. Policy correspondence.}
Since cumulative counts are determined by the original history, any feasible original policy can be represented as an admissible augmented policy by adding these counts, and any admissible augmented policy can be executed in the original process by maintaining them. In either direction, corresponding policies induce identical original state-action trajectory distributions and equal expected returns, with measurable feasible actions assigned on histories of probability zero where necessary. Consequently,
\begin{equation}
\sup_{\pi\in\Pi_K^{\mathrm{hist}}}
\mathbb E_{\tau\sim\pi}[R(\tau)]
=
\sup_{\bar\pi\in\bar\Pi^{\mathrm{hist}}}
\mathbb E_{\bar\tau\sim\bar\pi}[R(\bar\tau)].
\end{equation}

\noindent\textbf{Step 3. Markov optimality.}
Since the terminal reward is included in $r_1$, set ${\bar V_0^\star}=0$ and define {the optimal action-value and value functions of the augmented MDP} backward along the denoising trajectory, in the order $t=1,\ldots,N$
\begin{equation}
\begin{aligned}
{\bar Q_t^\star}(\bar s,a)
&=
\mathbb E\!\left[
 r_t+{\bar V_{t-1}^\star}(\bar s_{t-1})
 \mid\bar s_t=\bar s,\;a_t=a
\right],\\
{\bar V_t^\star}(\bar s)
&=
\max_{a\in\bar{\mathcal A}_t(\bar s)}
{\bar Q_t^\star}(\bar s,a).
\end{aligned}
\end{equation}
The value functions are bounded and measurable, and the finite nonempty action sets admit a measurable maximizing rule $f_t$, with ties broken by choosing the smaller action. For any $\bar\pi\in\bar\Pi^{\mathrm{hist}}$, backward induction from the terminal state $t=0$, using the tower property and the Markov transition, gives almost surely, for $1\le t\le N$,
\begin{equation}
\begin{aligned}
\mathbb E^{\bar\pi}\!\left[
\sum_{i=1}^{t}r_i\mid\bar h_t
\right]
&\le
\sum_{a\in\bar{\mathcal A}_t(\bar s_t)}
\bar\pi_t(a\mid\bar h_t)\,
{\bar Q_t^\star}(\bar s_t,a)\\
&\le{\bar V_t^\star}(\bar s_t),
\end{aligned}
\end{equation}
where $\bar h_t$ is the augmented history through time $t$. Under $f=(f_N,\ldots,f_1)$, equality holds throughout. Integrating over the initial state distribution shows that $f$ attains the optimal expected return over all admissible augmented history-dependent policies. Together with the policy correspondence, this proves the proposition. \end{proof}

\section{Additional Experiments}
\label{app:additional-results}

\subsection{Additional Qualitative Comparisons}
\label{app:qualitative-results}

{Figs.~\ref{fig:qualitative-extra} and~\ref{fig:qualitative-extra-wan21} present additional qualitative comparisons between MORCA and the baseline methods on Wan2.2 and Wan2.1, respectively, at three target acceleration ratios.}

\FloatBarrier

\subsection{Scheduler Overhead}
\label{app:scheduler-overhead}
\label{app:efficiency-accounting}

We additionally measure the time spent on scheduler execution and calculate its proportion of the total inference time. As shown in Table~\ref{tab:scheduler-overhead}, this proportion is approximately $1\%$, indicating acceptable overhead.

\begin{table}[H]
\centering
\setlength{\belowcaptionskip}{5pt}
\caption{\textbf{\fontencoding{T1}\selectfont Scheduler overhead on Wan2.2.} The scheduler introduces acceptable overhead.}
\label{tab:scheduler-overhead}
\fontencoding{T1}\fontsize{8}{9.2}\selectfont
\setlength{\tabcolsep}{10pt}
\renewcommand{\arraystretch}{1.2}

\begin{tabular}{cccc}
\toprule
\textbf{Target} & \textbf{Latency (s)} & \textbf{Scheduler time (s)} & \textbf{Proportion (\%)} \\
\midrule
$1.8\times$ & 549.19 & 3.95 & 0.72 \\
$2.4\times$ & 421.15 & 3.94 & 0.94 \\
$3.0\times$ & 326.93 & 3.93 & 1.20 \\
\bottomrule
\end{tabular}
\end{table}

\subsection{Detailed Comparisons with SeaCache}
\label{app:seacache-comparison}

\subsubsection{Schedule Visualizations at the Same Reuse Budget}
\label{app:matched-budget-schedules}

{Fig.~\ref{fig:c22-matched-budget} visualizes the cache schedules of MORCA and SeaCache under the same reuse budgets. MORCA learns schedules that differ from SeaCache's and achieves better generation fidelity in both examples.}

\begin{figure}[H]
\centering
\includegraphics[width=\linewidth]{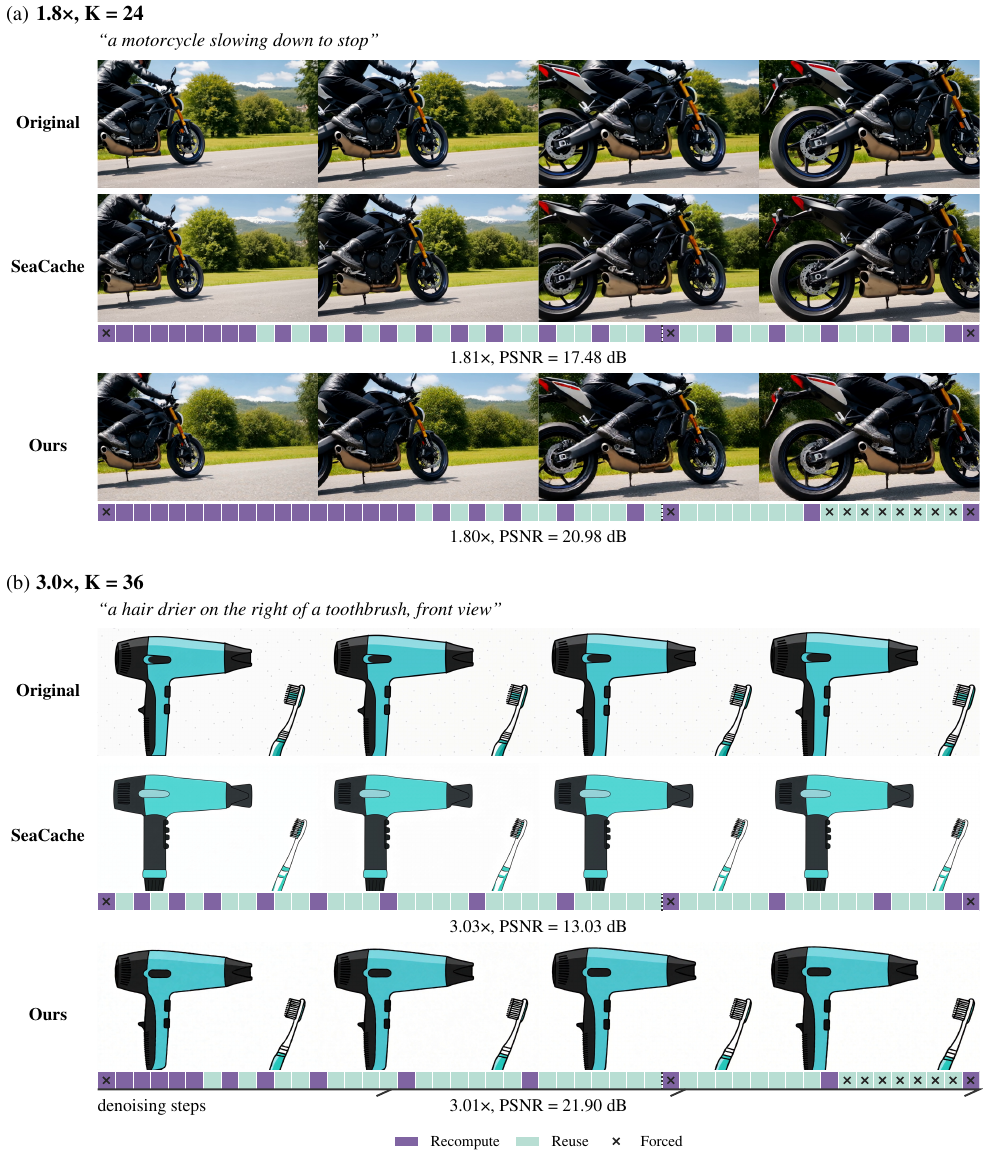}
\caption{\textbf{Visualization of schedules and corresponding generation results under the same reuse budget on Wan2.2.}}
\label{fig:c22-matched-budget}
\end{figure}
\FloatBarrier

\clearpage
\subsubsection{{Additional Target Acceleration Ratios}}
\label{app:seacache-speedups}

{We compare MORCA with SeaCache on Wan2.2 at nine target acceleration ratios ranging from $1.5\times$ to $3.5\times$ in increments of $0.25\times$. As shown in Fig.~\ref{fig:c21-acceleration}, MORCA consistently outperforms SeaCache in PSNR, SSIM, and LPIPS across all settings.}

\begin{figure}[H]
\centering
\includegraphics[width=\linewidth]{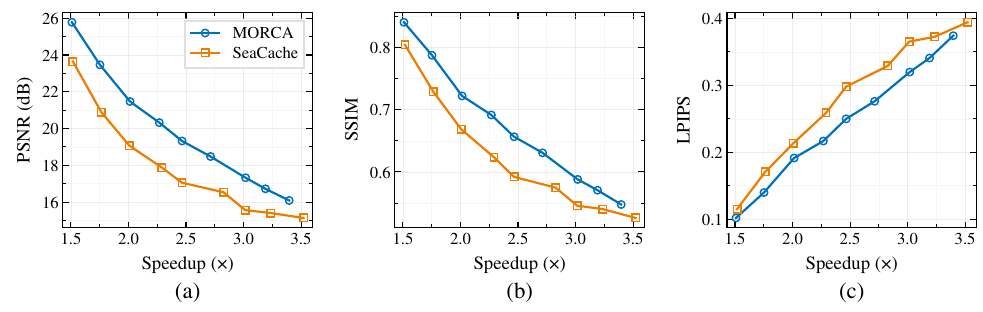}
\caption{{\textbf{Speed--quality trade-offs of MORCA and SeaCache on Wan2.2.} MORCA consistently achieves better fidelity at comparable speedups.}}
\label{fig:c21-acceleration}
\end{figure}

\FloatBarrier

\section{Limitations}
\label{app:limitations}

Due to differences in denoising dynamics across models, MORCA's learned scheduler cannot be directly transferred between models and requires separate training for each model. This introduces additional costs for collecting cache-scheduling trajectories and optimizing the scheduler. These costs may limit its practical adoption when training resources are constrained.

\clearpage
\begin{figure}[p]
\centering
\includegraphics[width=\linewidth]{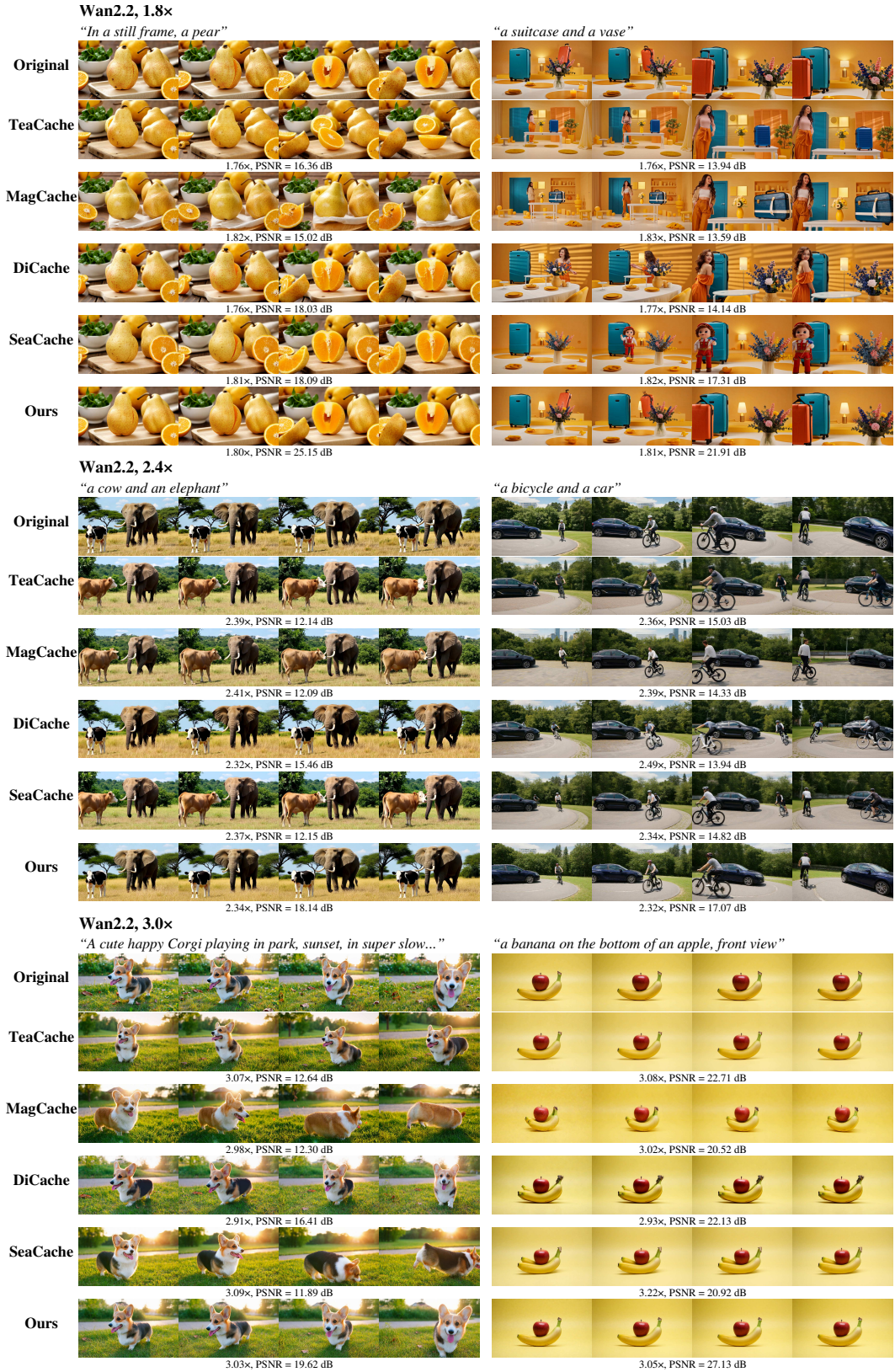}
\caption{{\textbf{Additional qualitative comparisons on Wan2.2.} From top to bottom: target acceleration ratios of $1.8\times$, $2.4\times$, and $3.0\times$.}}
\label{fig:qualitative-extra}
\end{figure}
\clearpage
\begin{figure}[p]
\centering
\includegraphics[width=\linewidth]{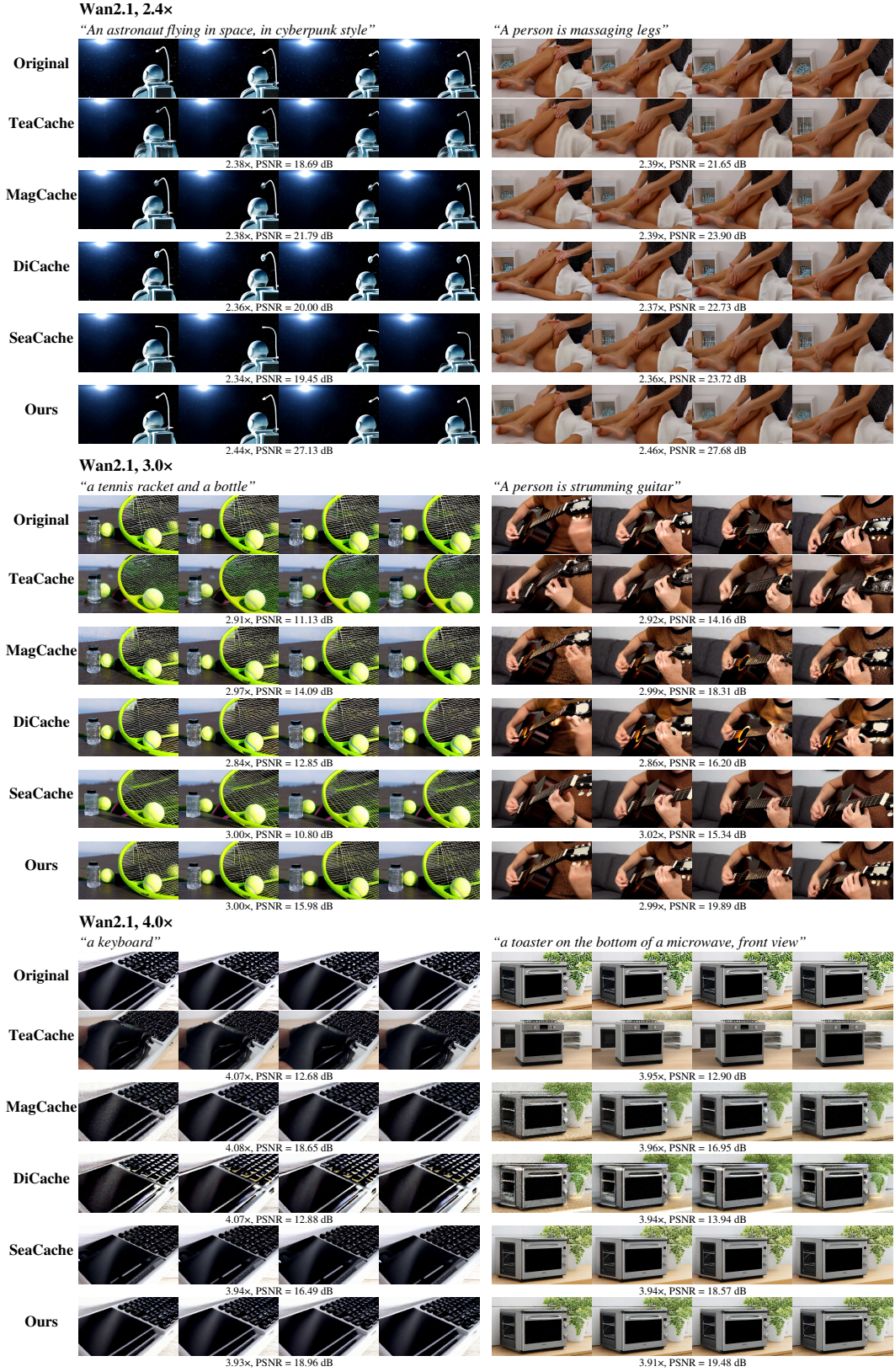}
\caption{{\textbf{Additional qualitative comparisons on Wan2.1.} From top to bottom: target acceleration ratios of $2.4\times$, $3.0\times$, and $4.0\times$.}}
\label{fig:qualitative-extra-wan21}
\end{figure}
\clearpage
\end{document}